\documentclass[journal]{IEEEtran}

\usepackage{amsmath}
\usepackage{array}
\usepackage{algpseudocode,algorithm,algorithmicx}
\usepackage{amsthm}
\usepackage{graphicx}
\usepackage{xcolor}
\usepackage{subcaption}
\usepackage{amssymb}
\usepackage[export]{adjustbox}
\usepackage{tabularx}
\usepackage{adjustbox}
\usepackage{multirow}
\usepackage{changepage}
\usepackage{diagbox}
\usepackage{easyReview}
\usepackage{cellspace}
\usepackage{cite}
 
\newtheorem{theorem}{Theorem}
\newtheorem{lemma}{Lemma} 
 
\newtheorem{remark}{Remark}

\newtheorem{definition}{Definition}

\definecolor{ao(english)}{rgb}{0.0, 0.5, 0.0}

\newcolumntype{C}[1]{>{\centering\arraybackslash}p{#1}}
\begin{document}
% \title{Harnessing complex dynamical systems at various scales through reservoir computing network}
%\title{An Anatomy of reservoir computing network and Its Linear Realization}
% \title{Revisiting reservoir computing networks and Linear Realization Theory}
% \title{Tractable Design of Reservoir Computing Networks Using Realization Theory}
\title{Interpretable Design of Reservoir Computing Networks using Realization Theory}
% \title{A realization-theoretic approaches to interpretable reservoir computing networks}

\author{Wei~Miao,~\IEEEmembership{Student Member,~IEEE,}
        Vignesh~Narayanan,~\IEEEmembership{Member,~IEEE,}
        and~Jr-Shin~Li,~\IEEEmembership{Senior~Member,~IEEE}% <-this % stops a space
\thanks{This research was supported in part by the National Science Foundation under the awards CMMI-1933976 and CMMI-1763070, and the NIH Grant R01GM131403-01. 
W. Miao, N. Vignesh and J.-S. Li are with the Department
of Electrical and Systems Engineering, Washington University in St. Louis, Saint Louis, MO, 63130 USA. e-mail: weimiao@wustl.edu, vignesh.narayanan@wustl.edu, jsli@wustl.edu.}}% <-this % stops a space
% \thanks{J.-S. Li is with the Department
% of Electrical and Systems Engineering, Washington University in St. Louis, Saint Louis, 
% MO, 63130 USA e-mail: jsli@wustl.edu}% <-this % stops a space
% \thanks{Manuscript received April 19, 2005; revised August 26, 2015.}}

\maketitle
% \listoftodos
%===========================   Abstract   ============================================== 
\begin{abstract}
The reservoir computing networks (RCNs) have been successfully employed as a tool in learning and complex decision-making tasks. Despite their efficiency and low training cost, practical applications of RCNs rely heavily on empirical design. In this paper, 
%we establish interpretable design principles to systematically choose the hyper-parameters and the hidden layer weights of an RCN for achieving guaranteed training accuracy in the applications involving time-series forecast. To this end, we interpret the training process of an RCN as approximating a time-delay embedding, and provide insights into making an informed choice of its hyper-parameters and activation functions for multi-step ahead forecasting of time-series. Furthermore, 
we develop an algorithm to design RCNs using the realization theory of linear dynamical systems. In particular, we introduce the notion of $\alpha$-stable realization, and provide an efficient approach to prune the size of a linear RCN without deteriorating the training accuracy. Furthermore, we derive a necessary and sufficient condition on the irreducibility of number of hidden nodes in linear RCNs based on the concepts of controllability and observability matrices. Leveraging the linear RCN design, we provide a tractable procedure to realize RCNs with nonlinear activation functions. Finally, we present numerical experiments on forecasting time-delay systems and chaotic systems to validate the proposed RCN design methods and demonstrate their efficacy.% in synthesizing RCNs. %, through the lens of Takens embedding theorem, 
\end{abstract}
%=================================================================================
%=================================================================================
\IEEEpeerreviewmaketitle
%=======================================================================================

%====================  Sec : Introduction  =============================
\section{Introduction}
%=======================================================================
The reservoir computing network (RCN) is a bio-mimetic computational tool that is increasingly used in a variety of applications to solve complex decision making problems \cite{jaeger2001echo, lukovsevivcius2012reservoir,TANAKA2019100}. Essentially, the RCN is a class of recurrent neural networks (RNNs), which is composed of one hidden layer, typically with a large number of sparsely interconnected neurons, and a linear output layer. In contrast to the classical RNN, a distinct feature of the RCN is that all of its connections in the hidden-layer are randomly pre-determined and fixed. Hence, the training process of the RCN involves only learning the weights of its linear output-layer in a supervised learning framework. %\replace{solve complex decision making tasks}{\remove{by the work presented}, or tackle complex decision making tasks

The existing supervised learning approach to training an RCN was proposed in \cite{jaeger2001echo}. Subsequently, the RCN was successfully employed for forecasting time-series with applications in finance \cite{lin2009short,GRIGORYEVA201459}, wireless communication \cite{jaeger2004harnessing}, speech recognition \cite{maass2002real}, and robot navigation \cite{antonelo2008event}. 
%More recently, the RCN was also configured for analyzing dynamical systems, such as distinguishing chaotic systems \cite{antonik2018using,lu2018attractor,carroll2018using}, learning coordinate-invariant properties (e.g., Lyapunov exponents) \cite{pathak2017using}, and designing an observer to infer unmeasured internal variables \cite{lu2017reservoir}. 
Notwithstanding its efficient training procedure, the major limitation of the RCN lies in the fact that its practical application relies heavily on empirical design of the hyper-parameters of the network, including its size \cite{lukovsevivcius2012practical}.%\remove{of the RCN and its computationally}{\replace{stems from}{\replace{On the other hand}{\textbf{}

Recently, there has been a renewed interest in developing tractable methods for designing neural networks that are suitably deployed in diverse scenarios \cite{doshi2017towards,lipton2018mythos}. In this context, deriving rigorous and systematic techniques to design neural networks, especially establishing principled strategies for selecting their hyper-parameters that yield a desired performance, is compelling but challenging. In this paper, we propose a tractable approach to design RCNs that warrant effective functioning for given datasets. In particular, we focus on the application of RCNs to learn dynamic models of dynamical systems from their time-series measurement data, and develop rigorous design principles to prune the number of hidden-layer nodes in RCNs without deteriorating the training performance. Leveraging the notions of controllability and observability matrices, we derive a necessary and sufficient condition on the irreducibility of RCNs with linear activation functions. This in turn results in an interpretable RCN pruning algorithm, where the RCNs' controllability and observability matrices inform on its size and irreducibility. Furthermore, we illustrate that the developed irreducible linear realization of the RCN not only sufficiently represents the underlying dynamics inherited in the time-series data, but also contributes to a tractable design of general RCNs with nonlinear activation functions.

The paper is organized as follows. In Section \ref{sec: prelim}, we provide a brief review of related works that motivate the need of our developments. In Section \ref{sec: apply Takens's theorem on RCN}, through tailoring existing results on the learnability of RNNs, we motivate our realization-theoretic RCN design principles by illustrating how an RCN achieves $\tau$-step ahead forecast of time-series associated with a dynamical system. In Section \ref{sec: linear realization}, we introduce realization-theoretic aspects from systems theory to facilitate a comprehensive RCN design, and then establish a necessary and sufficient condition on the irreducibility of a linear RCN that achieves desired training accuracy. This result in turn forms the basis to an educated design of RCNs with nonlinear activation functions. In Section \ref{sec: numerical results}, we present several results using an RCN to forecast time-series and learn chaotic systems to demonstrate the applicability of the proposed RCN design framework.

%=========================== SEC : Background and Related Works  ====================
\section{Related works and motivation}\label{sec: prelim}
%======================================================================
In this section, we briefly review the existing works on RCNs and point out the specific problems that we address in this paper.

The computational framework of an RCN and its training procedure were first proposed in \cite{jaeger2001echo}, where two conditions were hypothesized as requirements for successful applications of the RCNs - the \emph{echo state property} (ESP) and a general compactness assumption on the training signal. In addition, a mathematical definition of the ESP was also provided in \cite{jaeger2007optimization,grigoryeva2018universal}. Intuitively, the ESP implies that the state of the RCN is uniquely determined by its input history rather than the initial condition of the network. 
% A practical condition was proposed suggesting to design the connection matrix in the RCN with a spectral radius smaller than $1$.

Thereafter, several results explaining the principles of the RCN, especially in evaluating some of its features, such as the ESP \cite{buehner2006tighter,yildiz2012re}, the memory capacity \cite{jaeger2001short,Grigoryeva16,Marzen17}, and the stability \cite{boedecker2012information}, have been reported. In addition to investigating the fundamental properties of the RCN, multiple attempts addressing its design were also reported. The performance of the RCN in relation to the complexity of its network topology was analyzed in \cite{rodan2010minimum}. Using the interpretation of contracting maps, a discussion on the architecture of the RCN was presented in \cite{gallicchio2011architectural}. Supported by heuristic analyses, it was shown in \cite{kawai2019small} that the small-world network topology improved the performance (e.g., forecasting accuracy or memory capacity) of the RCN. In spite of the prescribed results, existing applications using RCNs rely on randomly generated connection matrices. Using the random matrix theory, an explanation on why such a design of the RCN, in general, leads to acceptable performance was presented \cite{zhang2011nonlinear}. 

Recently, there has been a rising tide of interest in analyzing RCNs, especially RCNs with linear activation functions, using control-theoretic approaches \cite{grigoryeva2020dimension,Bollt21}. For instance, the connectivity patterns of RCNs were studied in \cite{verzelli2021input} leveraging the concept of controllability matrix. Furthermore, it was proved in \cite{Gonon_2020} that the memory capacity of a linear RCN can be characterized by the rank of its controllability matrix. Nevertheless, due to the inherited randomness and usage, the design process of the RCNs are still based on empirical strategies, and designing an RCN with minimum size to ensure a desired performance is compelling but remains elusive.

In this work, we focus on establishing schematic design principles to realize RCNs with linear and nonlinear activation functions, which achieve guaranteed training accuracy. The main contributions of this paper include: (1) establishing the notion of $\alpha$-stable realization for designing the weight/connection matrix in an RCN with desired training accuracy for a given dataset; (2) devising an algorithm based on realization theory to prune the size of linear RCNs with quantifiable training accuracy; and (3) deriving tractable guidelines for configuring RCNs with nonlinear activation functions through the linear RCNs obtained via irreducible realization.

Here, we adopt the definition of the ESP as introduced in \cite{jaeger2001echo}, that is, an RCN is said to have the ESP if the state variables of the RCN are uniquely determined by the input history, regardless of the initial condition. Throughout this paper we denote $u[a;b]$ as a sequence with index starting from $a$ to $b$, where $a<b$, i.e., $u[a;b] := \{u[a], \ldots , u[b]\}$.

%===============  Sec: Main Results ==============================
\section{Role of Takens Embedding in RCN Frameworks}\label{sec: apply Takens's theorem on RCN}
%=================================================================
In this section, we provide details of the RCN dynamics and its training procedure. We tailor existing results on the learnability of RNNs, in particular, the Takens embedding theorem, to render a comprehensive analysis of how an RCN learns the underlying dynamics of a time-series. Based on the analyses, we discuss the applications of RCNs for the time series forecasting problem, which motivates the design of RCNs using linear realization theory in Section \ref{sec: linear realization}. We begin with a brief introduction to the Takens theorem and discuss its role in understanding RCNs.

%==================================================
\subsection{Takens theorem and its implication on time-series forecasting}
%==================================================
We consider a time-dependent variable $s(t)$, evolving on an $m$-dimensional manifold $M\subset \mathbb{R}^p$, following the dynamics $\dot{s}(t)=F_s(s(t)),$ where $F_s: M \to \mathbb{R}^p$ is a smooth vector field. Let $v:M\to \mathbb{R}$ be an observation function, and in practice, we measure a discrete sequence of observations, say $v[s(t_i)]$, where $t_i$ for $i = 0,1, 2,\ldots$ denoting the sampling instants. We can then define the propagation map $\phi:M \to M$ describing the flow of $s(t)$ at time $t_i$ by $s(t_{i+1})=\phi(s(t_i))$. Now, let $D(M) \subset C^2(M)$ denote the collection of functions such that for any $f \in D(M)$, $f:M\to M$ has an inverse function $f^{-1} \in C^2(M)$, where $C^2(M)$ denotes the class of functions over $M$ for which first- and second-order derivatives are continuous. Then, for the dynamical system describing the time evolution of $s(t)$, we have $\phi \in D(M)$. The Takens theorem can be stated as follows:

%==================  Takens Theorem   ==============================
\begin{theorem}[Takens theorem \cite{huke1993embedding}]
  Let $M$ be a compact manifold of dimension $m$. For pairs $(\phi, v)$ with $\phi\in D(M)$, and $v\in C^2(M, \mathbb{R})$, it is a generic property that the map $\Phi_{\phi, v, 2m+1}: M \to \mathbb{R}^{2m+1}$, defined by $\Phi_{\phi, v , 2m+1} (s) = (v(s), v(\phi(s)), \ldots , v(\phi^{2m}(s)))$ is an embedding, where `generic' means open and dense in $C^1$ topology. 
\end{theorem}
%=============================================================
% %==================== Figure =====================================
% \begin{figure*}[htb]
% \centering
% \includegraphics[width=0.75\linewidth]{Figure/Takens_prediction.pdf}
% \caption{A demonstration of using Takens theorem to forecast a time-series generated by a dynamical system on a compact manifold $M$ of dimension $m = 2$. Given $\phi\in D(M)$ and $v\in C^2(M, \mathbb{R})$, by Takens theorem, $\Phi_5(s) = (v(s), v(\phi(s)), \ldots , v(\phi^4(s))$ is an embedding. Let $v(0), \ldots , v(6)$ be the available observations. If one learns $\psi_{2m+1}$ based on the points in $\Phi(M)$, say $[v(0), \ldots , v(4)], [v(1), \ldots , v(5)], (v(2), \ldots , v(6))$, then one can use the evolution of $\psi_{2m+1}$ to forecast the next observation $v(7)$. Specifically if $v = I_d$, then $v(7) = s(7)$, which is the next state of $\phi$.}
% \label{fig: Takens's demo}
% \end{figure*}
% %=============================================================
From here on, we use $\Phi_{2m+1}$ as an abbreviation for $\Phi_{\phi, v, 2m+1}$, and call $2m+1$ as the `length' of the time-delay embedding $\Phi_{2m+1}$. Intuitively, Takens theorem states that for almost all pairs $(\phi, v)$ defined on a compact manifold $M$ of dimension $m$, there is an $1$-$1$ correspondence from $M$ to $\mathbb{R}^{2m+1}$ that preserves the structure of $M$. If Takens theorem holds, then by the definition of an embedding, the inverse function for the map $\Phi_{2m+1}$ is well-defined. Hence we can define a map $\psi_{2m+1}: = \Phi_{2m+1} \circ \phi \circ \Phi_{2m+1}^{-1}$, which describes the same dynamical system as $\phi$ does, under a coordinate change of $\Phi_{2m+1}$. The non-triviality of the construction of $\psi_{2m+1}$ is that it forecasts a new observation when provided with the time-delayed observations $(v(s), v(\phi(s)), \ldots , v(\phi^{2m}(s)))$. Namely, it holds that $\psi_{2m+1}(v(s), v(\phi(s)), \ldots , v(\phi^{2m}(s)))= (v(\phi(s)), v(\phi^2(s)), \ldots , v(\phi^{2m+1}(s)))$,
% \begin{align*}
%   & \psi_{2m+1}(v(s), v(\phi(s)), \ldots , v(\phi^{2m}(s))) \\
%   &= (v(\phi(s)), v(\phi^2(s)), \ldots , v(\phi^{2m+1}(s))),
% \end{align*}
which implies that if one learns the explicit representation of $\psi_{2m+1}$, then the new observation, i.e., $v(\phi^{2m+1}(s))$, can be predicted based on the historical observations, $(v(s),\ldots ,v(\phi^{2m}(s)))$. Specifically, if $v = I_d$, then $\psi_{2m+1}$ essentially predicts how the dynamical system defined by $\phi$ is evolving on $M$. 
%Figure \ref{fig: Takens's demo} demonstrates a simple case describing this application of Takens theorem for forecasting a time-series generated by a dynamical system. 
For additional details on Takens theorem, see \cite{huke1993embedding}, \cite{takens1981detecting}.

% %===========================================
% \begin{remark}\label{rem: Takens multi-output}
% Note that Takens theorem holds for arbitrarily long time-delays (greater than $2m+1$), and the generic embedding property is applicable for a multi-dimensional observation function $v$, i.e., $v\in C^2(M, \mathbb{R}^q)$ with $q>1$ as well.
% \end{remark}
% %===========================================
% Next, we provide an comprehensive analysis for understanding the RCN and its training process based on Takens theorem, which will serve as a guideline for the design of RCN for practical applications.

%=======================================
\subsection{The RCN dynamics and its training procedure}
%=======================================
In this part, we introduce the dynamics of the RCN and a sufficient condition for the RCN to possess the ESP. With the guarantee of the ESP and using Takens theorem, we show that an RCN can `learn' to forecast the data generated by a dynamical system on a compact manifold.%$D(M)$ 

Consider the RCN described by %with dynamics governed by
\begin{align}  
    &x[k+1] = (1-\alpha )x[k] + \alpha \sigma(Ax[k] + Bu[k]),\label{equ: state of RCN}\\
	&y[k] = Cx[k] = \sum_{i=1}^N c_i x_i[k], \label{equ: output of RCN}
\end{align}
where $x[k]\in \mathbb{R}^N$ denotes the state of the RCN and $N$ is the number of nodes; $u[k]\in \mathbb{R}^p$ is the input to the RCN; $y[k]\in \mathbb{R}^q$ is the output of the RCN; and $A\in \mathbb{R}^{N\times N}$ and $B\in \mathbb{R}^{N\times p}$ are pre-determined matrices denoting the connections in the RCN; $C\in \mathbb{R}^{q\times N}$ is the weight matrix in output-layer to be trained, with $c_i, i = 1, \ldots , N$ as its $i^{\text{th}}$ column; $\alpha \in (0, 1)$ is called the `leakage rate', and $\sigma$ is an activation function that is applied to a vector component-wise. A well-known result (see \cite{jaeger2001echo,grigoryeva2019differentiable}) for the RCN in \eqref{equ: state of RCN} to possess ESP is provided as the following lemma.
%============== Lemma ================
\begin{lemma}\label{lem: sufficient conditions for esp}
    Consider the dynamics of the RCN in \eqref{equ: state of RCN}. If $\sigma$ is Lipschitz continuous with the Lipschitz constant $L$, then the ESP holds if $\|A\|_2 < \frac{1}{L}$, where $\| \cdot \|_2$ denotes the matrix $2$-norm.
\end{lemma}
%====================================

Now, we illustrate the training process of the RCN. % modeled in \eqref{equ: state of RCN} and \eqref{equ: output of RCN}.} 
Given a time-series $u[0;T-1]$ for $T>1$, and a reference sequence $\tilde{y}[1;T]$, the RCN can be trained to predict the value of $\{\tilde{y}[k]\}$ for $k > T$. The canonical way to do this is to first select the size of the RCN (i.e., $N$), and randomly generate the matrices $A$ and $B$ of appropriate dimensions. Then, the sequence $u[0;T-1]$ is fed into the RCN dynamics \eqref{equ: state of RCN} as an input to generate a sequence of the RCN states $x[1;T]$. A fixed positive integer $w$ is selected as the `washout' length, and only the sequence after the $w^{\text{th}}$ step, i.e., $x[w; T]$, is collected. Finally, the coefficient in the output layer, $C$, is trained to minimize the error between the RCN output $y[w;T]$ and the reference sequence $\tilde{y}[w;T]$, that is,
$$C = \underset{C\in \mathbb{R}^{q\times N}}{\text{argmin}} \sum_{k=w}^T\big\|\,C \left[\begin{smallmatrix}
    x[k] \\ u[k]
\end{smallmatrix}\right] - \tilde{y}[k]\,\big\|_2^2.$$ 
Following this supervised learning procedure, when the input to the RCN is $u[k]$ (for $k>T$), the RCN outputs a value that approximates $\tilde{y}[k+1]$. In this sense, the training process enables the RCN to \emph{learn the underlying dynamics} of the reference sequence $\tilde{y}[k]$. 

In the following section, we explain in detail on how the RCN encodes the underlying dynamics of the reference sequence during the training process through the lens of Takens embedding theorem.

%==================================================
\subsection{Learning dynamics using an RCN}\label{subsec: learning dynamics by an RCN}
%================================================&= 
To begin with, we consider the task of $1$-step ahead forecast of a given time-series and explain how the training process enables the RCN to perform this task. For ease of exposition, we illustrate the idea with one-dimensional time-series, i.e., $u[k], y[k]\in \mathbb{R}$, and the framework is directly applicable to the multi-dimensional cases since the Takens theorem holds regardless of the dimension of the time-series.

\subsubsection{$1$-step ahead forecast}
Suppose a time-series $\{u[k]\} \subset \mathbb{R}$ is generated by a dynamical system on a compact manifold of dimension $m$. The sequence $u[0; T]$ is used as input to the RCN, and the $1$-step shifted sequence $u[1;T+1]$ is provided as the training reference. Denote the solution for the state equation in \eqref{equ: state of RCN} as $x_i[k] = \varphi_i(x[0], u[0], u[1], \ldots u[k-1])$, where $x_i[k]$ is the $i^{\text{th}}$ component of the vector $x[k]$, then by the uniqueness of the solution of a dynamical system, we have $x_i[k] =\varphi_i(x[0], u[0], u[1], \ldots u[k-1]) =\varphi_i(x[j], u[j], u[j+1], \ldots u[k-1])$
% \begin{align*}
% 	x_i[k] &=\varphi_i(x[0], u[0], u[1], \ldots u[k-1]) \\
% 	&=\varphi_i(x[j], u[j], u[j+1], \ldots u[k-1]), 
% \end{align*}
for any $j = 0, 1, 2,\ldots , k-1$.

As a result of Lemma \ref{lem: sufficient conditions for esp}, when $\|A\|_2 < \frac{1}{L}$, the RCN acquires the ESP. This implies that there exists a $k_0 \in \mathbb{N}$ such that, after the RCN evolves for $k_0$ steps, the effect of the initial condition on the solution trajectory becomes negligible. Therefore, for a fixed `washout' length $w > k_0$, we can define $\xi_{(i, w)}: \mathbb{R}^{w}\to\mathbb{R}^N$ such that 
\begin{align*}
    x_i[w+j] &= \varphi_i(x[j], u[j], \ldots , u[w+j-1]) \\
    &:= \xi_{(i, w)}(u[j], \ldots , u[w+j-1]),
\end{align*}
for all $j = 0, \ldots, T-w$, where $\xi_{(i, w)}$ can be treated as $\varphi_i$ taking historical data of length $w$, with the effect of the initial condition washed out. Recall that when training the RCN, we minimize the error between $y[w;T]$ and $\tilde{y}[w;T]$. Hence, training the RCN is equivalent to finding $c_i$'s such that
\begin{align}
  &u[w+j] \approx y[w+j] = \sum_{i=1}^N c_ix_i[w+j] \nonumber\\
  &= \sum_{i=1}^N c_i \varphi_i(x[j], u[j], u[j+1], \ldots , u[w+j-1]) \nonumber\\
  &= \sum_{i=1}^N c_i \xi_{(i, w)}(u[j], u[j+1], \ldots , u[w+j-1])\label{equ: 1-1}, 
\end{align}
for $j = 0, \ldots , T-w$. 

We observe from \eqref{equ: 1-1} that when the RCN is endowed with the ESP, the training procedure is equivalent to learning the map between $u[w+j]$ and the window of historical data $u[j;w+j-1]$. The existence of such a map is guaranteed by Takens theorem. In particular, with $w$ in \eqref{equ: 1-1} greater than $2m+1$, there exists a smooth map $\psi_w: \mathbb{R}^{w}\to \mathbb{R}^{w}$ such that $\psi_{w}(u[j;w+j-1])  = u[j+1;w+j]$. Therefore, training an RCN can also be interpreted as % equivalent to 
approximating the map $\psi_w$ using nonlinear functions $\xi_{(i, w)}$ for $i = 1, \ldots, N$.

% \cred{Define the last component of $\psi_w$ as $G_w:\mathbb{R}^{p\times w}\to \mathbb{R}^p$, satifsfying $G_w(u[j;w+j-1]) = u[w+j]$. Then the training process of the RCN in \eqref{equ: 1-1} is equivalent to finding a vector $c = (c_1, \ldots, c_N)$ such that
% \begin{align}
% 	G_w(&u[j; w+j-1]) = u[w+j] \nonumber\\
% 	&\approx \sum_{i=1}^N c_i \xi_{(i, w)}(u[j], u[j+1], \ldots , u[w+j-1]). \label{equ: 1-step forecasting}
% \end{align}
% This can be interpreted as a nonlinear function approximation, where a linear combination of $\xi_{(i, w)}$ is used to approximate the unknown map $G_w$. (Need to remove this part. -WM)}

\subsubsection{Multi-step ahead forecasting}
Based on the idea of $1$-step ahead forecast, we can extend the RCN to accomplish $\tau$-step ahead forecast. Specifically, if the training reference is set to be $\tilde{y}[w;T] = u[w+\tau-1;T+\tau-1]$, then similar to \eqref{equ: 1-1}, training the RCN is equivalent to finding $c_i$'s such that $u[w+j+\tau-1] \approx \sum_{i=1}^N c_i \xi_{(i, w)}(u[j;w+j-1])$
% \begin{equation*}
%     u[w+j+\tau-1] \approx \sum_{i=1}^N c_i \xi_{(i, w)}(u[j;w+j-1]),
% \end{equation*}
for $j = 0, \ldots, T-w$. A similar argument to $1$-step ahead forecasting holds for $\tau$-step forecasting since the output of the RCN can be configured to approximate $\psi_w^\tau = \psi_w \circ \cdots \circ \psi_w$ ($\tau$ times) such that $\psi_{w}^\tau(u[j;w+j-1])  = u[j+\tau;w+j+\tau-1]$.
Hence, in this way, the RCN training can be viewed as approximating %equivalent to approximating 
the last component of $\psi_w^{\tau}$ by using $\sum_{i=1}^N c_i \xi_{(i,w)}$. 

\begin{remark}
	The above explanation on multi-step ahead forecasting based on time-delay embedding provides a much better understanding on why the RCN successfully forecasts a chaotic system as reported in \cite{antonik2018using,carroll2018using}. In particular, for a chaotic system with an attractor, such as the Lorenz system, the sequence $\{u[k]\}$ will eventually lie in a compact manifold of low dimension. Hence, Takens theorem together with the analysis presented above on time-delay embedding can be directly applied.
\end{remark}

\subsubsection{Learning dynamics of observation sequences}
In addition to forecasting the input sequence $u[k]$ for $1$-step/multi-steps ahead, where $\tilde{y}[k] = u[k]$, we can also configure RCN to forecast observation sequences that can be expressed as functions of $u[k]$.
% the RCN can be trained to forecast \cred{other observation sequences that can be expressed as functions of $u[k]$}. 
In particular, if $\tilde{y} = h(u)$, where $h$ is a smooth function, then $\tilde{y}[k]$ can be represented solely by the history of $u[k]$, say $\tilde{y}[w+j] = G(u[j; w+j-1])$ since $u[k]$ is uniquely determined by a window of historical data. Therefore, the training process of the RCN, in this case, can be interpreted as using $\sum_{i=1}^N c_i\xi_{(i, w)}$ to approximate the nonlinear mapping $G$ (similar to \eqref{equ: 1-1}). This illustrates the ability of the RCN to forecast a wide variety of time-series by simply changing the training reference $\tilde{y}[k]$.

\subsubsection{Choices of the activation function}
When the activation function $\sigma$ is Lipschitz continuous with the Lipschitz constant $L$, the condition $\|A\|_2 < \frac{1}{L}$ is sufficient to ensure the ESP for \eqref{equ: state of RCN} as a result of Lemma \ref{lem: sufficient conditions for esp}. For instance, the commonly used hyperbolic tangent function, $\tanh$, has a Lipschitz constant $L = 1$ so that we need $\|A\|_2 < 1$ to ensure the ESP for \eqref{equ: state of RCN}. Since the above explanation on how RCNs learn dynamical systems does not post any restriction on the activation function $\sigma$, we indeed have much freedom on the choice of $\sigma$. In fact, the activation function can be as simple as a \emph{linear function}. 

The trade-off between the RCN performance and the choice of activation function for the RCN is well-documented. Although it is believed, in general, that nonlinear activation functions perform better than linear ones, it is reported in the literature (see \cite{verstraeten2010memory,verzelli2019echo,inubushi2017reservoir,verstraeten2007experimental}) that this conclusion is indeed based on a case-by-case study. One major advantage %of the major advantages 
of using linear activation functions, as we shall see in Section \ref{sec: linear realization}, is that a linear activation function enables a thorough analysis of the RCN using system-theoretic tools, and allows for an explicit design of the RCN, which can eventually be used as a baseline for designing an RCN with a nonlinear activation function, e.g., $\tanh$ or $\mathrm{sigmoid}$, as widely used in the literature.% \cred{memory capacity paper}

%====================== Realization Theory ====================
\section{Realization Theory For the RCN design}\label{sec: linear realization}
%====================================================&= 
In this section, we present a realization-theoretic framework for systematic design of RCNs. We illustrate the main idea and conduct the analysis for RCNs with linear activation functions. Specifically, we first show that for given input-output sequences $u[0;T-1] \subset \mathbb{R}^p$ and $\tilde{y}[1;T] \subset \mathbb{R}^q$, there exists an RCN that approximates the input-output relation in the data. Then, we provide a systematic scheme to prune the size of RCNs while maintaining the same training error. At the end of this section, we illustrate how these results can be carried over to the design of RCNs with general nonlinear activation. 
% \replace{to reduce the size of the RCN until we find the minimum size that achieves a given performance}{\replace{used for realization of the RCN}{

% ==========================================
\subsection{Realization of RCNs with linear activation function}\label{sec:linear_irreducible}
% ==========================================
% \cred{(Read this section again! May be better organized!)} \\
%\cred{We elaborate that directly finding such a minimum realization is also an explicit RCN design approach, which further characterizes the dimension of Takens embedding in Corollary \ref{lem: 2}. }
We begin by introducing the notion of \emph{realization} from systems theory \cite{de2000minimal} and defining it in the context of the RCN, which will form the basis of the proposed RCN design framework.

%=================== Defn ======================
\begin{definition}[Realization of linear systems]%\cb{[Realization of RCNs with linear activation]}
	Given two sequences $u[0; T-1] \subset \mathbb{R}^p$ and $\tilde{y}[1; T]\subset \mathbb{R}^q$, we say that the triplet, $\tilde{A}\in \mathbb{R}^{N\times N}$, $\tilde{B}\in \mathbb{R}^{N\times p}$, and $\tilde{C}\in \mathbb{R}^{q\times N}$, is an \textit{N-dimensional $\epsilon$-error realization} of the pair $(u[0;T-1], \tilde{y}[1;T])$ if the linear system,
	\begin{equation}
		\begin{aligned}
			x[k+1] &= \tilde{A}x[k] + \tilde{B}u[k],\\
			y[k] & = \tilde{C} x[k],
		\end{aligned}\label{equ: lienar system}
	\end{equation}
	% \remove{with $x[0] = 0$,} 
	satisfies $\sqrt{\sum_{k=1}^{T}\|y[k] - \tilde{y}[k]\|_2^2} \leq  \epsilon$. Furthermore, if $\epsilon = 0$, then such a realization is called an \textit{N-dimensional perfect realization}.%\cred{(This is a customized definition that helps presenting this work. -WM)}
\end{definition}
%==============================================
For simplicity, we will refer to an $N$-dimensional realization using $(\tilde{A}, \tilde{B}, \tilde{C})_N$, and this will denote the linear dynamical system in \eqref{equ: lienar system}. To facilitate the delineation between two realizations, we introduce Markov parameters, equivalent and irreducible realizations as follows.
%=================== Defn ======================
\begin{definition}[Markov parameter]
	The $k^{\text{th}}$ Markov parameter of a realization $(\tilde{A}, \tilde{B}, \tilde{C})_N$ is  a matrix of real numbers $\gamma_k\in \mathbb{R}^{q\times p}$ defined by $\gamma_k = \tilde{C}\tilde{A}^k\tilde{B}$.
\end{definition}
%===============================================
%=================== Defn ======================
\begin{definition}[Equivalent realizations]
	Two realizations $(A_1, B_1, C_1)_{N_1}$ and $(A_2, B_2, C_1)_{N_2}$ are said to be \textit{equivalent} if $\gamma_k^{(1)}=\gamma_k^{(2)}$ holds for all $k = 0, 1, 2,\ldots $ where $\gamma_k^{(1)} =C_1A_1^kB_1$ and $\gamma_k^{(2)}= C_2A_2^k B_2$. 
\end{definition}
%=============================================
%=================== Defn ======================
\begin{definition}[Irreducible realization] \label{def: irreducible.realization}
	A realization $(\tilde{A}, \tilde{B}, \tilde{C})_N$ is said to be \textit{irreducible} if there exists no equivalent realization $(\hat{A}, \hat{B}, \hat{C})_{\bar{N}}$ with $\bar{N} < N$.
\end{definition}
%===========================================

\begin{remark}
	In the literature of control systems, Definition \ref{def: irreducible.realization} is referred as the `minimal realization' if $(\tilde{A}, \tilde{B}, \tilde{C})$ is perfect \cite{de2000minimal}. Since in this work, we consider the reduction of $\epsilon$-error realizations of RCNs, Definition \ref{def: irreducible.realization} is named as `irreducible realization' to avoid ambiguity.
\end{remark}

% \begin{remark}
% \label{rem:MIMO}
% 	Note that if the input sequence is one-dimensional and the output sequence is multi-dimensional, i.e., $u[0; T-1] \subset \mathbb{R}$, $y[1;T]\subset \mathbb{R}^{k_1}$ with $k_1 > 1$, then they are realizable by a single-input multi-output (SIMO) linear system defined in the form of \eqref{equ: lienar system}, where $\tilde{A}\in \mathbb{R}^{N\times N}$, $\tilde{B}^{N\times 1}$ and $\tilde{C}^{k_1\times N}$. In this case, the SIMO realization of the input-output pair can be considered as a combination of $k_1$ SISO realizations given by $(\tilde{A}, \tilde{B}, \tilde{C}_j)_N$, $j = 1, \ldots, k_1$, where $\tilde{C}_j$ are $j^{\text{th}}$ row of $\tilde{C}_j$. If the input sequence is multi-dimensional, i.e., $u[0; T-1] \subset \mathbb{R}^{k_2}$ with $k_2>1$, then this corresponds to a `multi-dimensional realization' problem. For multi-input linear systems, finding the minimal state-space realization is a case-by-case study (see for instance \cite{Antoniou88minimal, Mentzelopoulou91ndim}). Therefore, in this section, we only consider realizing pairs of one-dimensional sequences, for which finding the minimal state-space realization can be achieved in a systematic fashion.
% \end{remark}

%Next, we elaborate that directly finding a minimum realization $(\tilde{A}, \tilde{B}, \tilde{C})_N$ is equivalent to training an RCN using matrices $\tilde{A}$ and $\tilde{B}$. To see the connection between
Next, we will describe the RCN training procedure through the use of a realization $(\tilde{A}, \tilde{B}, \tilde{C})_N$, and then establish the realization framework tailored for the RCN that explicitly accounts for the ESP, an important and necessary property for the functioning of the RCN. 

Consider the RCN as given in \eqref{equ: state of RCN} with a linear activation function, e.g., $\sigma$ is the identity function, given by
\begin{equation*}
	x[k+1] = [(1-\alpha )I + \alpha A]x[k] + \alpha B u[k],
\end{equation*}
where $I$ is the identity matrix of appropriate dimension. Let $\tilde{A} := (1-\alpha)I + \alpha A\in \mathbb{R}^{N\times N}$ and $\tilde{B} = \alpha B \in \mathbb{R}^{N\times p}$, then the RCN dynamics can be expressed %compactly represented 
as 
\begin{equation}\label{equ: RCN dynamics rewritten}
	x[k+1] = \tilde{A}x[k]+\tilde{B}u[k].
\end{equation}

Now. we introduce the connection between the RCN training procedure and the notion of realization theory. Specifically, the first step of RCN training is to fix the dimension $N$, the leakage rate $\alpha$, and the randomly generated matrices $\tilde{A}$ and $\tilde{B}$. Then, training the output layer of the RCN constructed using $\tilde{A}$ and $\tilde{B}$ is equivalent to finding $\tilde{C}^*$ such that
\begin{equation}\label{equ: RCN training C}
	\begin{aligned}
		\tilde{C}^* & = \underset{\tilde{C}\in \mathbb{R}^{q\times N}}{\text{argmin}\:}E[(\tilde{A}, \tilde{B}, \tilde{C})_N],
	\end{aligned}
\end{equation}
where $E[(\tilde{A}, \tilde{B}, \tilde{C})_N] := \sqrt{\sum_{k=1}^T \|\tilde{y}[k] - \tilde{C}\tilde{A}^k\tilde{B}u[k]\|_2^2}$ denotes the training error of the realization $(\tilde{A}, \tilde{B}, \tilde{C})_N$.

Let us consider an $N_1$-dimensional $\epsilon$-error realization, denoted $(\tilde{A}_1, \tilde{B}_1, \tilde{C}_1)_{N_1}$, and suppose $(\tilde{A}_2, \tilde{B}_2, \tilde{C}_2)_{N_2}$ is the irreducible equivalent realization of $(\tilde{A}_1, \tilde{B}_1, \tilde{C}_1)_{N_1}$ with $N_2 \leq N_1$. Then, it holds that
\begin{align*}
	\underset{\tilde{C}\in \mathbb{R}^{q\times N_2}}{\text{min}\:} E[(\tilde{A}_2, \tilde{B}_2, \tilde{C})_{N_2}] &\leq E[(\tilde{A}_2, \tilde{B}_2, \tilde{C}_2)_{N_2}]\\
	  &=  E[(\tilde{A}_1, \tilde{B}_1, \tilde{C}_1)_{N_1}]< \epsilon,
\end{align*}
which implies that if we train the RCN constructed using the $N_2$-dimensional matrices $\tilde{A}_2\in\mathbb{R}^{N_2\times N_2}$ and $\tilde{B}_2\in\mathbb{R}^{N_2\times p}$, the training error will be bounded above by $\epsilon$. Therefore, finding the irreducible equivalent realization to a given RCN enables quantifying a smaller size of the RCN that provides a desired training accuracy. 

To adopt this realization-theoretic idea for the RCN design, an additional constraint has to be imposed on the matrix $A$ in order for the RCN to be equipped with the ESP. To achieve this, we propose the notion of $\alpha$-stable realization.
%=================== Defn ======================
\begin{definition}[$\alpha$-stable realization]
	Given $\alpha \in (0, 1]$, a realization $(\tilde{A}, \tilde{B}, \tilde{C})_N$ is called an $\alpha$-stable realization if $\|\tilde{A}\|_2 < \alpha$.
\end{definition}
%=============================================

For instance, we know, by Lemma \ref{lem: sufficient conditions for esp}, that an RCN in \eqref{equ: state of RCN} using a linear activation function $\sigma$ with the Lipschitz constant $L=1$ possesses the ESP when $\|A\|_2 < 1$.
% For instance, for an activation function, say $\sigma$ is linear in \eqref{equ: state of RCN} with $L=1$, from Lemma \ref{lem: sufficient conditions for esp} we know that an RCN possesses the ESP when $\|A\|_2 < 1$. 
Therefore, in this case, for any $\alpha \in (\frac{1}{2}, 1)$, having $\|\tilde{A}\|_2 < 2\alpha - 1$ in \eqref{equ: RCN dynamics rewritten} is sufficient to guarantee the ESP. This is because
\begin{align*}
	\|A\|_2 &= \|\frac{1}{\alpha } (\tilde{A} - (1-\alpha) I)\|_2 	\leq  \frac{1}{\alpha }[\|\tilde{A}\|_2 + (1-\alpha )]\\
	&< \frac{1}{\alpha }[2\alpha - 1 + 1-\alpha ] = 1.
\end{align*}
As a consequence, finding a $(2\alpha -1)$-stable realization $(\tilde{A}, \tilde{B}, \tilde{C})_N$ will ensure that the corresponding RCN possesses the ESP.

Therefore, in the remainder of this section, we will consider %a given \cb{pair of} input-output sequences $u[0; T-1]$ and $y[1; T]$, and 
a fixed leakage rate $\alpha \in (\frac{1}{2}, 1)$, so that $2\alpha-1\in (0,1)$. 
For simplicity, we use `stable realization' in place of `$(2\alpha -1)$-stable realization'. %to find a $(2\alpha -1)$-stable realization $(\tilde{A}, \tilde{B}, \tilde{C})_N$

\subsection{The irreducible stable realization}
%==============================================
Theoretically, if there exists one stable realization that achieves a desired training error for a given input-output sequence, one can construct many different realizations with the same training error. A consequent question of paramount practical importance to ask is how to prune the size of a realization as much as possible while maintaining the training error tolerance of the given stable realization. The answer to this question is pertinent to the concept of fundamental properties of a control system.%\replace{We begin by introducing related concepts from linear dynamical systems theory.}{

% ============= Definition 6 ===========
\begin{definition}[Controllability and observability matrices]\label{def: cont.obse.matrices}
	For a linear time-invariant dynamical system $(\tilde{A}, \tilde{B}, \tilde{C})_N$ as modeled in \eqref{equ: lienar system}, the controllability and the observability matrices are defined by
	\begin{align*}
		W_N = [\tilde{B}, \tilde{A}\tilde{B}, \ldots, \tilde{A}^{N-1}\tilde{B}],
\quad \text{and} \quad		G_N = \begin{bmatrix}
			\tilde{C} \\ \tilde{C}\tilde{A} \\ \vdots \\ \tilde{C}\tilde{A}^{N-1}
		\end{bmatrix},
	\end{align*}
respectively.
\end{definition}

Controllability and observability properties then lead to the characterization of a irreducible realization (see \cite{silverman1971realization}).
\begin{lemma}\label{lem: irreducible realization without constraint}
	A realization $(\tilde{A}, \tilde{B}, \tilde{C})_N$ is irreducible if and only if the pair $(\tilde A, \tilde B)_N$ is controllable and the pair $(\tilde A, \tilde C)_N$ is observable, i.e., $\mathrm{rank}\: (W_N) = N$ and $\mathrm{rank}\: (G_N) = N$.
\end{lemma}
% \begin{remark}
Note that Lemma \ref{lem: irreducible realization without constraint} poses no constraints on the matrix norm of the realization, and thus an irreducible realization may be unstable. As a result, modifications have to be made in order to construct a stable irreducible realization resulting in an RCN with ESP. In the following, we develop a systematic scheme to construct an equivalent stable realization of RCN with reduced size.

%\remove{Based on the notion of controllability and observability matrices, we develop a systematic scheme to find an equivalent stable realization of RCN with reduced size.}\remove{Therefore, to incorporate this result into the RCN design, it should be tailored for the purpose of designing stable realizations so that the resulting RCN possesses the ESP.}

% ========== Lemma 3 ============
\begin{lemma}\label{lem: realization reduction controllability}
	Given an $N$-dimensional $\alpha$-stable realization $(\tilde{A}, \tilde{B}, \tilde{C})_N$, if $\mathrm{rank}\: W_N = \bar{N} < N$, then there exists an $\bar{N}$-dimensional $\alpha$-stable realization $(\bar{A}, \bar{B}, \bar{C})_{\bar{N}}$ that is equivalent to $(\tilde{A}, \tilde{B}, \tilde{C})_N$.
\end{lemma}
\begin{proof}
	Because $\mathrm{rank}\: W_N = \bar{N} < N$, let $v_1, \ldots , v_{\bar{N}}$ be an orthonormal basis of $\mathcal{R}(W_N)$, the column space of $W_N$, and select $v_{\bar{N}+1}, \ldots, v_{N}$ such that $v_1, \ldots, v_N$ forms an orthonormal basis of $\mathbb{R}^N$. Also, we denote $V_1 = [v_1, \ldots , v_{\bar{N}}]$, $V_2 = [v_{\bar{N}+1}, \ldots , v_N]$ and $V = [V_1, V_2]$.  Note that by construction $V$ is an orthonormal matrix so that $V^{-1} = V^{\intercal}$. Therefore, we have
	\begin{align*}
		V^{-1}\tilde{A}V = V^{\intercal}\tilde{A}V = \begin{bmatrix}
			V_1^{\intercal} \tilde{A}V_1  &V_1^{\intercal}\tilde{A}V_2\\
			V_2^{\intercal} \tilde{A} V_1 & V_2^{\intercal}\tilde{A}V_2
		\end{bmatrix}.
	\end{align*}
	Note that each column of $\tilde{A}V_1$ lies in $\mathcal{R}(V_1)$. Since columns of $V_2$  are in the orthogonal complement of $\mathcal{R}(V_1)$ by construction, it holds that $V_2^{\intercal}\tilde{A}V_1 = 0$ . Therefore, $V^{-1}\tilde{A}V$ can be re-written as
	\begin{align}\label{equ: pf-lem-2-1}
		V^{-1}\tilde{A}V = \begin{bmatrix}
			V_1^{\intercal} \tilde{A}V_1  &V_1^{\intercal}\tilde{A}V_2\\
			0 & V_2^{\intercal}\tilde{A}V_2
		\end{bmatrix} :=  \begin{bmatrix}
			{A}_{11}  & {A}_{12}\\
			0 & {A}_{22}
		\end{bmatrix}.
	\end{align}
	Moreover, since every column of $\tilde{B}$ lies in $\mathcal{R}(W_N)$, we have $V_2^{\intercal}\tilde{B} = 0$ by the construction of $V_2$ so that
% Similarly,  $V^{-1}\tilde{B}=V' \tilde{B}$, and \cred{since $V^{-1}\tilde{B} \in \mathcal{R}(V_1)$ ($\mathrm{rank}\: W_N = \bar{N}$),} we have
	\begin{align}\label{equ: pf-lem-2-2}
		V^{-1}\tilde{B} = V^{\intercal}\tilde{B} = \begin{bmatrix}
		V_1^{\intercal}\tilde{B} \\ 0
		\end{bmatrix} := \begin{bmatrix}
		B_1 \\ 0
		\end{bmatrix}.
	\end{align}
	Now we denote $\tilde{C}V = [{C}_1, {C}_2]$, where ${C}_1$ consists of the first $\bar{N}$ columns of $\tilde{C}V$, and ${C}_2$ is formed by the remaining $N-\bar{N}$ columns. Then, it can be observed that $(V^{-1}\tilde{A}V, V^{-1}\tilde{B}, \tilde{C}V)_N$ is equivalent to $(\tilde{A}, \tilde{B}, \tilde{C})_N$, since $$\tilde{C}V (V^{-1}\tilde{A} V)^{k} V^{-1}\tilde{B} = \tilde{C}\tilde{A}^k \tilde{B}$$ holds for all $k = 0, 1, \ldots$. Furthermore, due to the structure provided by \eqref{equ: pf-lem-2-1} and \eqref{equ: pf-lem-2-2}, it can be verified that 
	\begin{align*}
		\tilde{C}V &(V^{-1}\tilde{A} V)^{k} V^{-1}\tilde{B} = [C_1, C_2] \left( \begin{bmatrix}
			A_{11}  & A_{12}\\
			0 & A_{22}
		\end{bmatrix}\right)^k \begin{bmatrix}
		B_1 \\ 0
		\end{bmatrix}\\
		&= [C_1, C_2] \begin{bmatrix}
			A_{11}^k  & *\\
			0 & A_{22}^k
		\end{bmatrix}\begin{bmatrix}
		B_1 \\ 0
		\end{bmatrix}= C_1A_{11}^k B_1,
	\end{align*}
	which implies that $(A_{11}, B_1, C_1)_{\bar{N}}$ is an $\bar{N}$-dimensional realization that is equivalent to $(\tilde{A}, \tilde{B}, \tilde{C})_N$.  Now it remains to show that $(A_{11}, B_1, C_1)_{\bar{N}}$ is $\alpha$-stable, given that $(\tilde{A}, \tilde{B}, \tilde{C})_N$ is $\alpha$-stable. By the definition of matrix $2$-norm, we have
	\begin{equation}
		\|A_{11}\|_2 = \underset{\substack{x\in \mathbb{R}^{\bar{N}} \\ \|x\|_2 = 1}}{\text{sup}}\: x^{\intercal}A_{11}x = \underset{\substack{x\in \mathbb{R}^{\bar{N}} \\ \|x\|_2 = 1}}{\text{sup}}\: x^{\intercal}V_1^{\intercal}\tilde{A}V_1x. \label{equ: pf-lem-2-4}
	\end{equation}
	Let $y = V_1x$, then $y$ is a vector in $\mathbb{R}^N$ satisfying
	\begin{align*}
		\|y\|_2^2 &= y^{\intercal}y = (x_1v_1 + \cdots + x_{\bar{N}}v_{\bar{N}})^{\intercal}(x_1v_1 + \cdots + x_{\bar{N}}v_{\bar{N}})\\
		&= x_1^2 + \cdots + x_{\bar{N}}^2 = \|x\|_2^2.
	\end{align*}
	Hence, \eqref{equ: pf-lem-2-4} can be bounded by $\|A_{11}\|_2 =  \underset{\substack{y\in \mathbb{R}^N \\ \|y\|_2 = 1}}{\text{sup}}\: y^{\intercal}\tilde{A} y \leq  \|\tilde{A}\|_2 < \alpha$, which concludes the proof.
\end{proof}

From a dual perspective, we also have the following lemma regarding the observability matrix.

% ========= Lemma 4 ==========
\begin{lemma}\label{lem: realization reduction observability}
	Given an $\alpha$-stable realization $(\tilde{A}, \tilde{B}, \tilde{C})_N$, if $\mathrm{rank}\: G_N = \bar{N} < N$, then there exists an $\bar{N}$-dimensional $\alpha$-stable realization $(\bar{A}, \bar{B}, \bar{C})_{\bar{N}}$ equivalent to $(\tilde{A}, \tilde{B}, \tilde{C})_N$.
\end{lemma}

The proof is omitted since it is similar to the proof of Lemma \ref{lem: realization reduction controllability}. With the help of Lemmas \ref{lem: realization reduction controllability} and \ref{lem: realization reduction observability}, we develop an explicit criterion on characterizing the irreducible stable realization for the RCN given in \eqref{equ: RCN dynamics rewritten}.

% ========== Theorem 2 ============
\begin{theorem}\label{thm:irreducible stable realization}
	A stable realization $(\tilde{A}, \tilde{B}, \tilde{C})_N$ is irreducible, i.e., there exists no equivalent stable realization $(\bar{A}, \bar{B}, \bar{C})_{\bar{N}}$ with $\bar{N} < N$, if and only if $\mathrm{rank}\: W_N G_N = N$.
\end{theorem}
\begin{proof}
	We prove the theorem by proving the contraposition, i.e., $(\tilde{A}, \tilde{B}, \tilde{C})_N$ is not irreducible if and only if $\mathrm{rank}\: W_N G_N < N$.

	On the one hand, if $(\tilde{A}, \tilde{B}, \tilde{C})_N$ is not irreducible, then by Lemma \ref{lem: irreducible realization without constraint}, either $\mathrm{rank}\: W_N < N$ or $\mathrm{rank}\: G_N < N$. Therefore, $$\mathrm{rank}\: (W_NG_N) \leq \min (\mathrm{rank}\: (W_N), \mathrm{rank}\: G_N) < N.$$
	On the other hand, if $\mathrm{rank}\: (W_NG_N) < N$, then, by Sylvester's rank inequality, it holds that $$\mathrm{rank}\:(W_N) + \mathrm{rank}\: (G_N) - N \leq \mathrm{rank}\: (W_N G_N) < N,$$ which implies that $\mathrm{rank}\: W_N + \mathrm{rank}\: G_N < 2N$ and that either $\mathrm{rank}\:(W_N) < N$ or $\mathrm{rank}\: (G_N) < N$. Hence, from Lemma \ref{lem: realization reduction controllability} (or Lemma \ref{lem: realization reduction observability}), it holds that $(\tilde{A}, \tilde{B}, \tilde{C})_N$ is not irreducible.% which concludes the proof considering the results of
\end{proof}

As a consequence of Lemma \ref{lem: realization reduction controllability}, Lemma \ref{lem: realization reduction observability}, and Theorem \ref{thm:irreducible stable realization}, the procedure for finding {the} irreducible stable realization that is equivalent to a given stable realization $(\tilde{A}, \tilde{B}, \tilde{C})_N$ is described in Algorithm \ref{algo:1}, where $\mathsf{Orth}(A)$ returns an orthonormal basis of $\mathcal{R}(A)$, and $\mathrm{dim}\:\tilde{A}$ returns the dimension of the matrix $\tilde{A}$.

\begin{remark}
	It is worthwhile to mention that as proved in \cite{Gonon_2020}, a linear RCN attains maximal memory capacity when its weight matrices $(A, B)$ constitutes a full-rank controllability matrix. As a consequence, Algorithm \ref{algo:1} not only returns an irreducible linear realization of RCN, but also provides an RCN that reaches maximum memory capacity.
\end{remark}

\begin{algorithm}[h]
	\caption{Minimal stable realization}  \label{algo:1}
	\begin{algorithmic} 
		\Function{Minimal stable realization}{$\tilde{A}, \tilde{B}, \tilde{C}$}
			\State\textbf{Initialize}: Compute $W_N, G_N$ for $(\tilde{A}, \tilde{B}, \tilde{C})_N$
			\While{$\mathrm{rank}\:(W_NG_N) < N$}
				\If {$\mathrm{rank}\: (W_N) < N$}
					\State $V_1 = \mathsf{Orth}\:(W_N)$.
				\Else
					\State $V_1 = \mathsf{Orth}\:(G_N^{\intercal})$.
				\EndIf
				\State $\tilde{A} = V_1^{\intercal}\tilde{A}V_1$, $\tilde{B} = V_1^{\intercal}\tilde{B}$, $\tilde{C} = \tilde{C}V_1$,
				\State $N = \mathrm{dim}\: \tilde{A}$.
				\State Compute $W_N, G_N$ for $(\tilde{A}, \tilde{B}, \tilde{C})_N$.
			\EndWhile
			\State \Return{$(\tilde{A}, \tilde{B}, \tilde{C})_N$.} 
		\EndFunction  
	\end{algorithmic}  
\end{algorithm}

\begin{remark}
	Each iteration in Algorithm \ref{algo:1} consists of computing the eigen-decomposition of the controllability or the observability matrix, which has a time-complexity of $O(N^3)$. In the worst case, Algorithm \ref{algo:1} may take $N$ iterations to terminate, which results in a total time-complexity of $O(N^4)$. Nevertheless, in all our numerical experiments, we observe that the number of iterations for Algorithm \ref{algo:1} to terminate is of order much smaller than $N$, i.e., around $10$ to $20$ iterations for the cases when $N = 500, 1000$ or even $N = 2000$, so that we empirically expect that the average time-complexity of Algorithm \ref{algo:1} is $O(N^3)$. On the other hand, model selection procedures for designing an initial learning model for the given data typically involves evaluating the performance of models with various hyper-parameters and choosing the model that yields the best result \cite{aras2016new, rodan2010minimum}. One of the main features of the proposed approach is that the results of Theorem 2 can be used to evaluate the irreducibility of the RCN and Algorithm 1 can be used to prune the RCN size, irrespective of how the initial RCN model is selected.
\end{remark}
%==============================================
\subsection{Further implications of irreducible stable realizations}
% \subsection{\cb{Indications beyond the minimal linear realization}}
%==============================================
\label{subsec: RCN design using the irreducible linear realization}
The development of irreducible realization in the previous section has a lot to offer for designing linear and nonlinear RCNs, as well as understanding the underlying dynamics in the training dataset. We start by explaining how the size of the irreducible realization is related to Takens embedding, which provides a criterion to characterize the complexity of the underlying dynamics determined by $u[0;T-1]$ and $\tilde{y}[1;T]$.

From the theory of linear dynamical system (see \cite{brockett2015finite}), it is a known fact that for any $N$-dimensional realization $(\tilde{A}, \tilde{B}, \tilde{C})_N$, there exists an invertible matrix $P\in \mathbb{R}^{N\times N}$ such that $(P^{-1}\tilde{A} P, P^{-1}\tilde{B}, \tilde{C}P)_N$ is in the observable canonical form  given by 
\begin{align} 
P^{-1}\tilde{A} P &= \renewcommand{\arraystretch}{1.1}\begin{pmatrix}
 0 		& I_q &   & \\ 
 0 & 0 & \ddots\\
 \vdots & \vdots & \ddots & \ddots& \\
 0	& 0 & \cdots & 0& I_q\\
 -\theta_0 I_q  	& -\theta_1 I_q& \cdots	& \theta_{N-2}I_q & -\theta_{N-1}I_q
\end{pmatrix}, \nonumber\\
P^{-1}\tilde{B} &= \begin{pmatrix}
\gamma_0 \\ \vdots \\ \gamma_{N-1}
\end{pmatrix}, \quad \tilde{C}P = \begin{pmatrix}
I_q, & 0, &\cdots, & 0
\end{pmatrix}, \label{equ: canonical form}
\end{align}
where $\theta_j, j = 0, 1, \ldots , N-1$ are arbitrary constants and $\gamma_j, j = 0, 1, \cdots , N-1$ are the first $N$ Markov parameters of the realization $(\tilde{A}, \tilde{B}, \tilde{C})_N$. It is not hard to verify that for any invertible matrix $P\in \mathbb{R}^{N\times N}$, we have $(P^{-1}\tilde{A} P, P^{-1}\tilde{B}, \tilde{C}P)_N$ to be equivalent to $(\tilde{A}, \tilde{B}, \tilde{C})_N$. Therefore, without loss of generality, we assume that $(\tilde{A}, \tilde{B}, \tilde{C})_N$ is in the observable canonical form as in \eqref{equ: canonical form}. In this case, the dynamics of each component of $x[k]$  is given by %when the realization $(\tilde{A}, \tilde{B}, \tilde{C})_N$ is canonical:
\begin{align*}
	x[k+1] &=\begin{pmatrix}
	x_2[k] + \gamma_0u[k]\\
	\vdots\\
	x_N[k] + \gamma_{N-2}u[k]\\
	\sum_{i=0}^{N-1} \theta_i x_i[k] + \gamma_{N-1}u[k]
	\end{pmatrix}, \quad\tilde{y}[k] = x_1[k],
\end{align*}
where $x_i[k]$ is the $i^{\text{th}}$ component of $x[k]$. Therefore, we have
\begin{align}
	&\tilde{y}[k+N] = x_1[k+N] = x_2[k+N-1] + \gamma_0u[k+N-1] \nonumber\\
	&= \cdots =x_N[k+1] + \sum_{i=0}^{N-1}\gamma_i u[k+N-1-i]. \label{equ: output time-delay embedding}
\end{align}
When the RCN possesses the ESP, the effect of $x_N[k+1]$ on $\tilde{y}[k+N]$ is negligible, so that from \eqref{equ: output time-delay embedding}, $\tilde{y}[k+N]$ is determined by $u[k], \ldots, u[k+N-1]$, which is the input data history of size $N$. From Takens theorem, the time evolution on a compact manifold of dimension $m$ can be represented by any time-delay embedding longer than $2m+1$. Therefore, if a linear realization describes the underlying dynamics determined by $u[0;T-1]$ and $\tilde{y}[1;T]$ perfectly, the dimension of such realization satisfies $N >  2m+1$ by Takens theorem. On the other hand, if a linear realization of size $N$ attains a training error of $\epsilon$, then it implies that there exists a dynamical system evolving on a manifold of dimension $\frac{1}{2}(N-1)$ that approximately represents the underlying dynamics up to $\epsilon$-error. This analysis provides a bound on the size of an RCN representing the dynamics of the underlying dynamical system generating the given input-output data sequences.

In addition, it is worth mentioning that since we are using linear dynamics to approximate the map $\psi_{2m+1}$ in Takens embedding, the bound on $m$ mentioned above can be improved through the use of a nonlinear realization. It is intuitive to argue that using a nonlinear realization (e.g., with $\tanh$ or $\mathrm{sigmoid}$ as the activation function) to design an RCN may result in better approximation compared to a linear realization of the same dimension. However, the explicit solution of RCN dynamics with a nonlinear activation function is in general unavailable. In this case, the linear realization of the RCN offers a guideline towards designing a nonlinear RCN realization. A reasonable design approach for the RCN with a nonlinear activation function is to first design a linear realization, say $(\tilde{A}, \tilde{B}, \tilde{C})_N$, for the given input-output data sequences; then construct an RCN with nonlinear activation function using the matrices $A$, and $B$ as in \eqref{equ: state of RCN} through 
\begin{align*}
	A = \frac{1}{\alpha} (\tilde{A} - (1-\alpha) I), \quad B = \frac{1}{\alpha} \tilde{B},
\end{align*}
and train the readout layer again.

%===================================================
\section{Numerical experiments} \label{sec: numerical results}
%===================================================
In this section, we present several numerical examples to illustrate the developed tractable realization-theoretic approach with training error guarantees, for which the irreducible size of RCNs with respect to desired training errors can be explicitly quantified. Based on linear stable realizations, we further design nonlinear RCNs with the canonical $\tanh$ or $\mathrm{sigmoid}$ activation functions that achieve desired training error performance. %\cred{(do we compare the results using linear and nonlinear activation functions?)(Yes. See Example C, Tables I and II. -WM)}

%====================================================
\subsection{The irreducible realization of time-delay systems}\label{subsec: time-delay system}
%====================================================
In this example, we design an RCN using the proposed approach for forecasting a time-delay system to elucidate the intimate connection between the irreducible linear realization and Takens embedding.

We first introduce how the training data was generated and how the RCN was trained. For a fixed time-delay $\tau$, we randomly picked $u[0], \ldots, u[\tau-1]$ under a uniform distribution on $[-1, 1]$. Then, we completed the sequence of $u$ by setting $u[k+\tau] = u[k]$ for $k = 0, \ldots, T-1$. In this way, the dynamical system governing the sequence $u[0;T]$ is a $\tau$-step time-delay system. After generating the sequence $u[0;T]$, we used $u[0; T-\tau]$ as the input and $u[\tau, T]$ as the reference output to train an RCN with linear activation function.

In this experiment, we varied $\tau$ from $1$ to $50$. For each $\tau$, we randomly generated a $N_0$-dimensional $0.001$-error realization with $N_0>\tau$ and computed its irreducible realization using Algorithm 1. Then, using the irreducible realization, we forecast the sequence $\tau$-step ahead for {$2000$} time steps. When training RCNs, we fixed training length as $T = 1000$, and leakage rate as $\alpha = 0.9$.

Figure \ref{fig: Nmin_moving_average} demonstrates the size of irreducible linear realization $N_{\mathrm{min}}$ with respect to $\tau$ with $N_0$ selected as $N_0 = 4\tau$. Each point in the figure is the average plus/minus the standard deviation of $10$ independent experiments under the same setup. Figure \ref{fig: Nmin_moving_average} shows that using Algorithm \ref{algo:1}, we can trim a large RCN (dashed green line) into much smaller size (solid blue line) with the same performance. This result is consistent with our analysis in Section \ref{subsec: RCN design using the irreducible linear realization} that the minimum size of an RCN realization can be used as a criterion to quantify the length of time-delay embedding (dashed red line) associated with the training dataset. Figure \ref{fig: Nmin_moving_average_ts_error} provides the averaged mean squared errors (MSE) of the irreducible linear realization with respect to $\tau$.
\begin{figure}[htb]
  \centering
  \adjustbox{minipage=1em,valign=t}{\subcaption{}\label{fig: Nmin_moving_average}}%
  \begin{subfigure}[b]{0.42\linewidth}
    \centering\includegraphics[width=\linewidth]{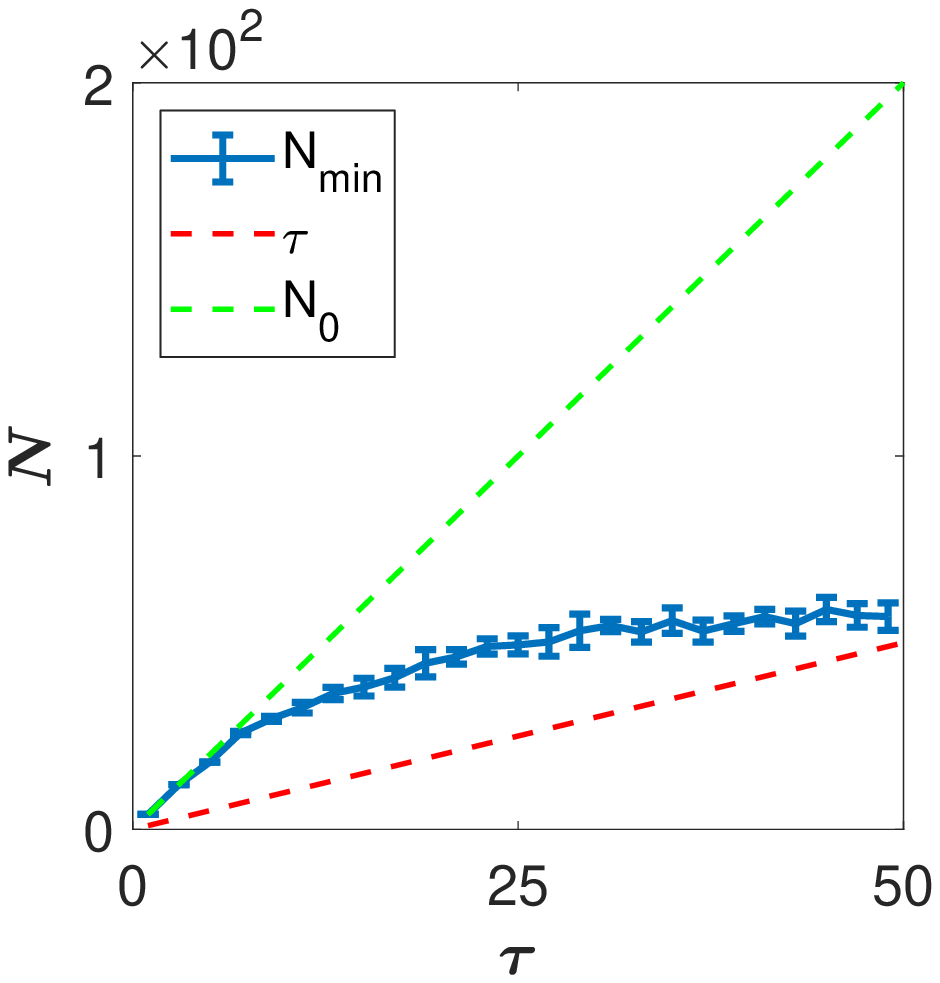} 
  \end{subfigure}
  \hspace{10pt}
  \adjustbox{minipage=1em,valign=t}{\subcaption{}\label{fig: Nmin_moving_average_ts_error}}%
  \begin{subfigure}[b]{0.42\linewidth}
    \centering\includegraphics[width=\linewidth]{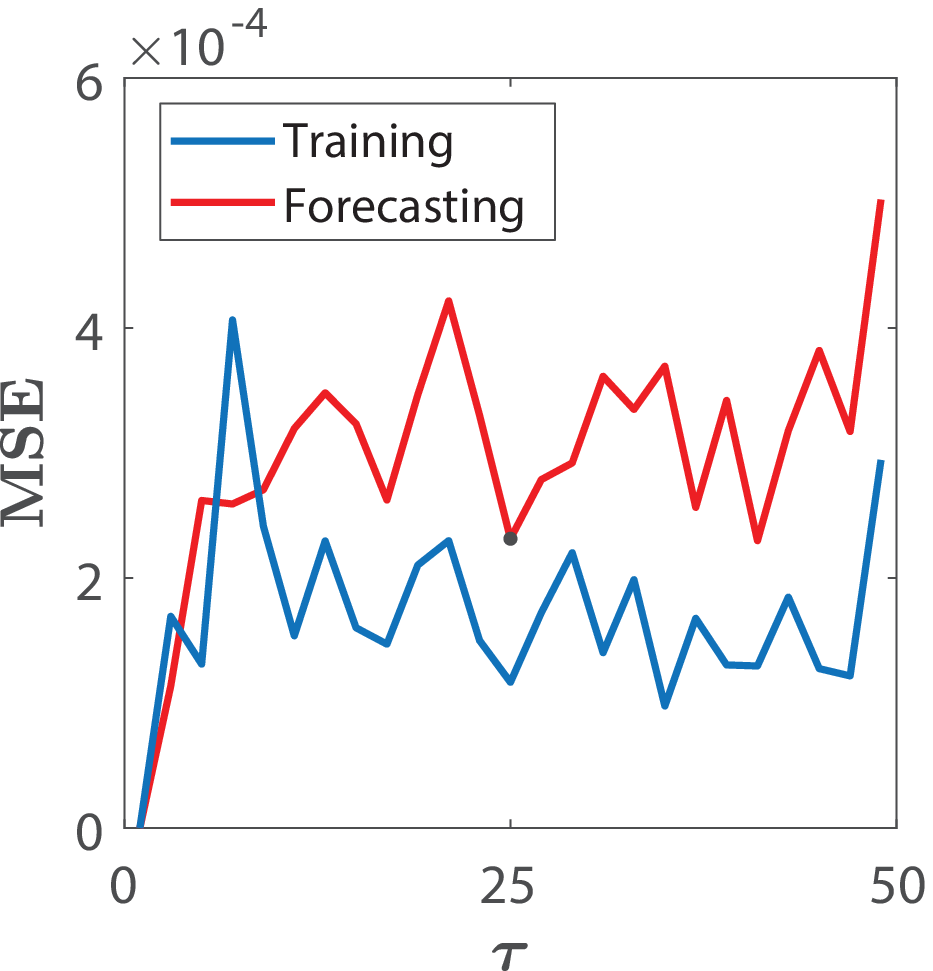}
  \end{subfigure}
  \caption{Results of designing RCNs to forecast a time-delay system. (a) The size of irreducible $0.001$-error realization versus length of time-delay. (b) The training and forecasting MSE of the irreducible $0.001$-error realization versus length of time-delay.}
\end{figure}
\subsection{Time-evolution forecast for chaotic systems}
\label{subsec: Rossler}
\begin{figure*}[htb]
  \centering
  \adjustbox{minipage=1.3em,valign=t}{\subcaption{}\label{fig: k_step_Rossler}}%
  \begin{subfigure}[b]{0.65\linewidth}
    \centering\includegraphics[width=1\linewidth]{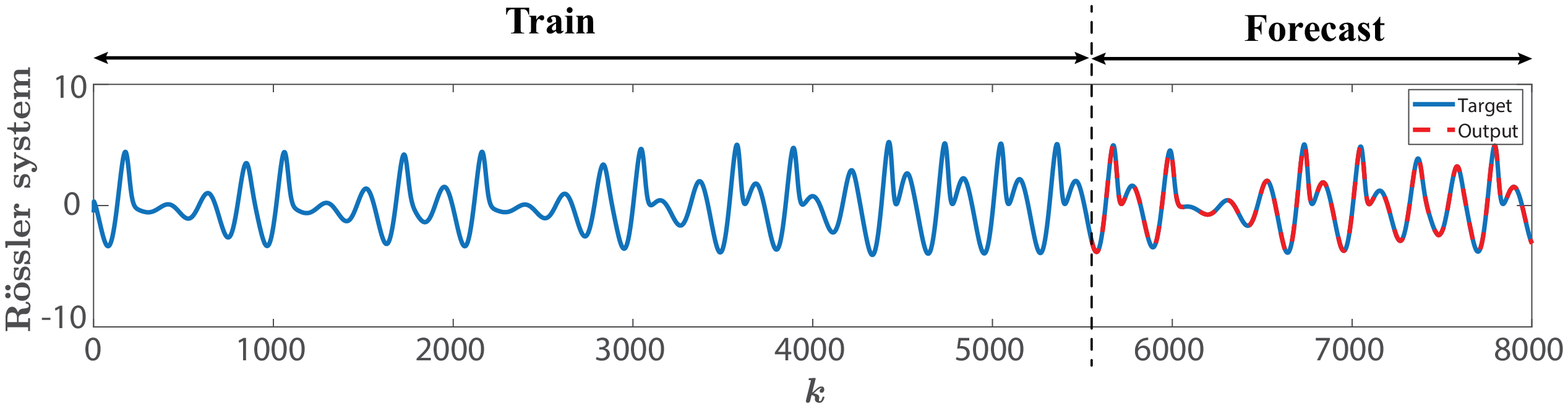} 
  \end{subfigure}
  
  \hspace{0pt}
  \adjustbox{minipage=1em,valign=t}{\subcaption{}\label{fig: k_step_tanh}}%
  \begin{subfigure}[b]{0.2\linewidth}
    \centering\includegraphics[width=\linewidth]{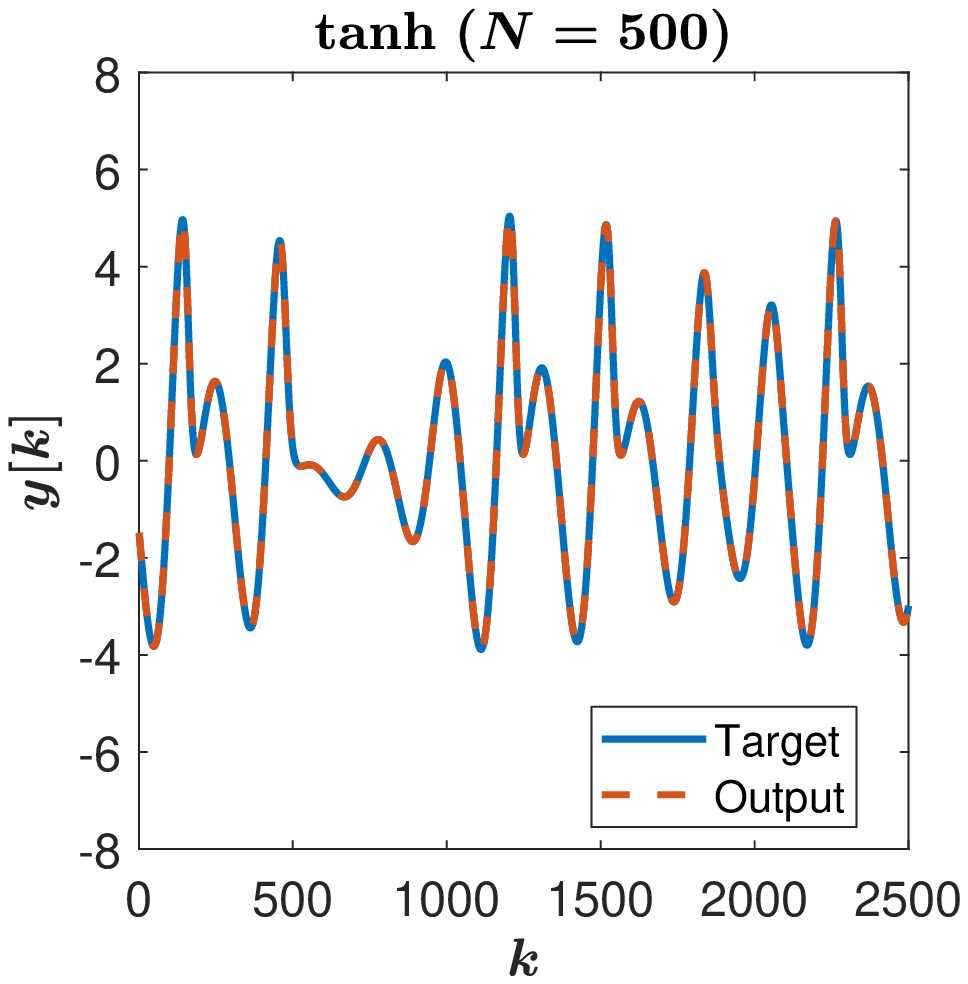} 
  \end{subfigure}
  \hspace{10pt}
  \adjustbox{minipage=1em,valign=t}{\subcaption{}\label{fig: k_step_sweep_alpha}}%
  \begin{subfigure}[b]{0.2\linewidth}
    \centering\includegraphics[width=\linewidth]{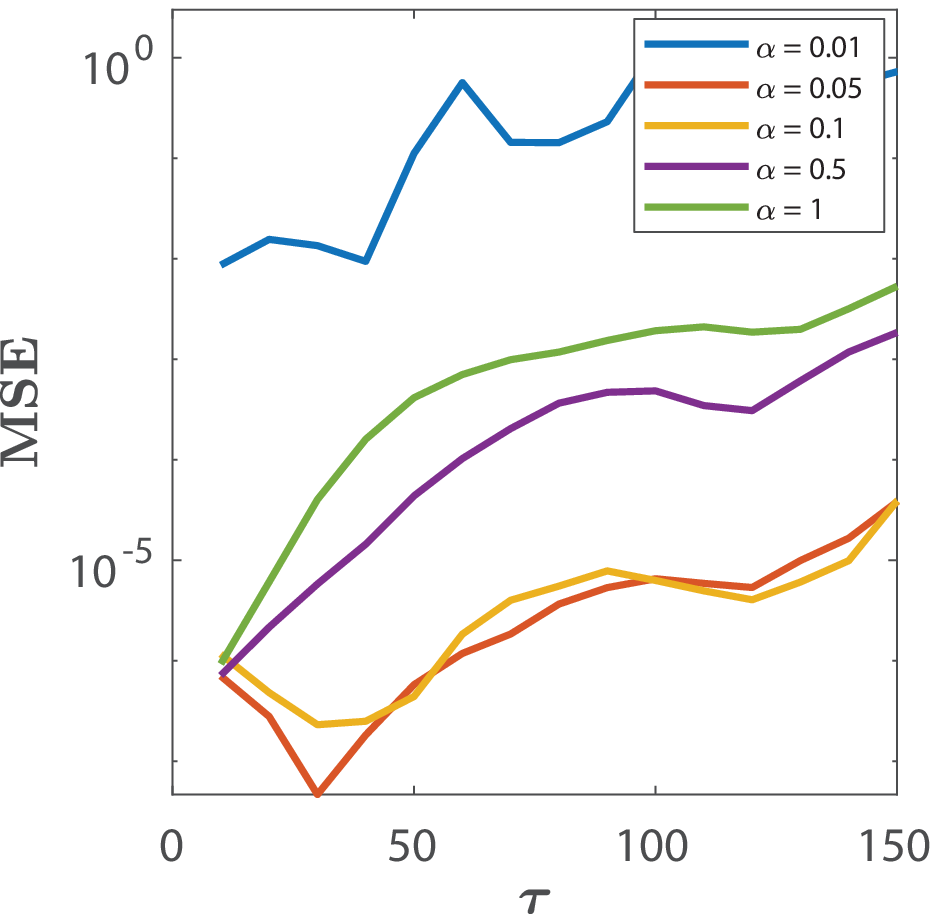} 
  \end{subfigure}
  \hspace{10pt}
  \adjustbox{minipage=1em,valign=t}{\subcaption{}\label{fig: k_step_logMSE}}%
  \begin{subfigure}[b]{0.2\linewidth}
    \centering\includegraphics[width=\linewidth]{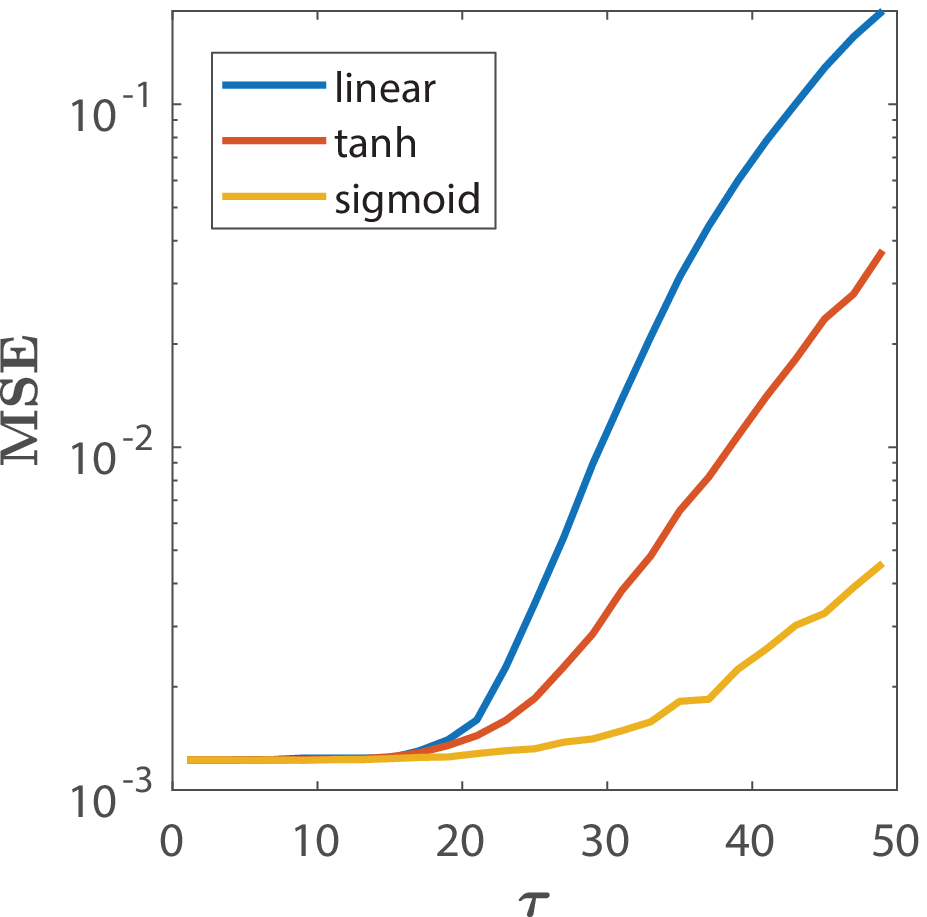} 
  \end{subfigure}

  \hspace{10pt}
  \adjustbox{minipage=1em,valign=t}{\subcaption{}\label{fig: k_step_linear_N=500}}%
  \begin{subfigure}[b]{0.2\linewidth}
    \centering\includegraphics[width=\linewidth]{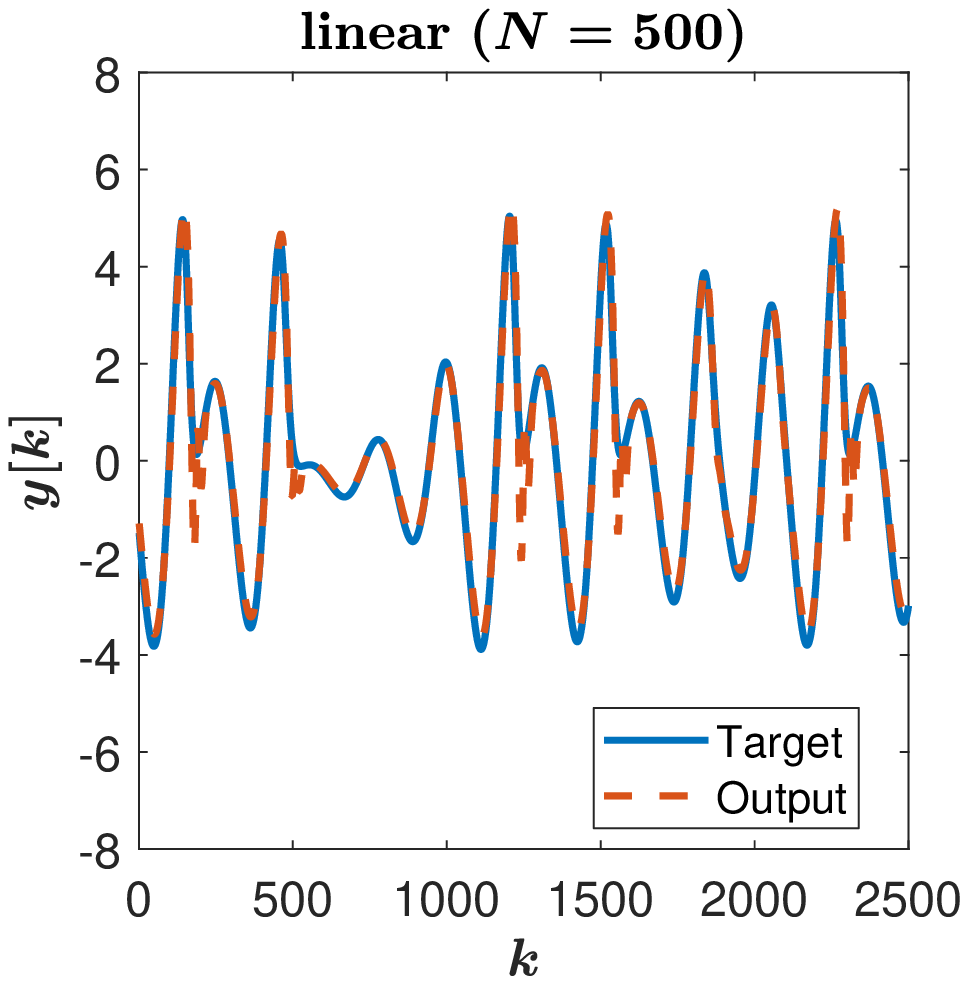} 
  \end{subfigure}
  \hspace{10pt}
  \adjustbox{minipage=1em,valign=t}{\subcaption{}\label{fig: k_step_linear_N=31}}%
  \begin{subfigure}[b]{0.2\linewidth}
    \centering\includegraphics[width=\linewidth]{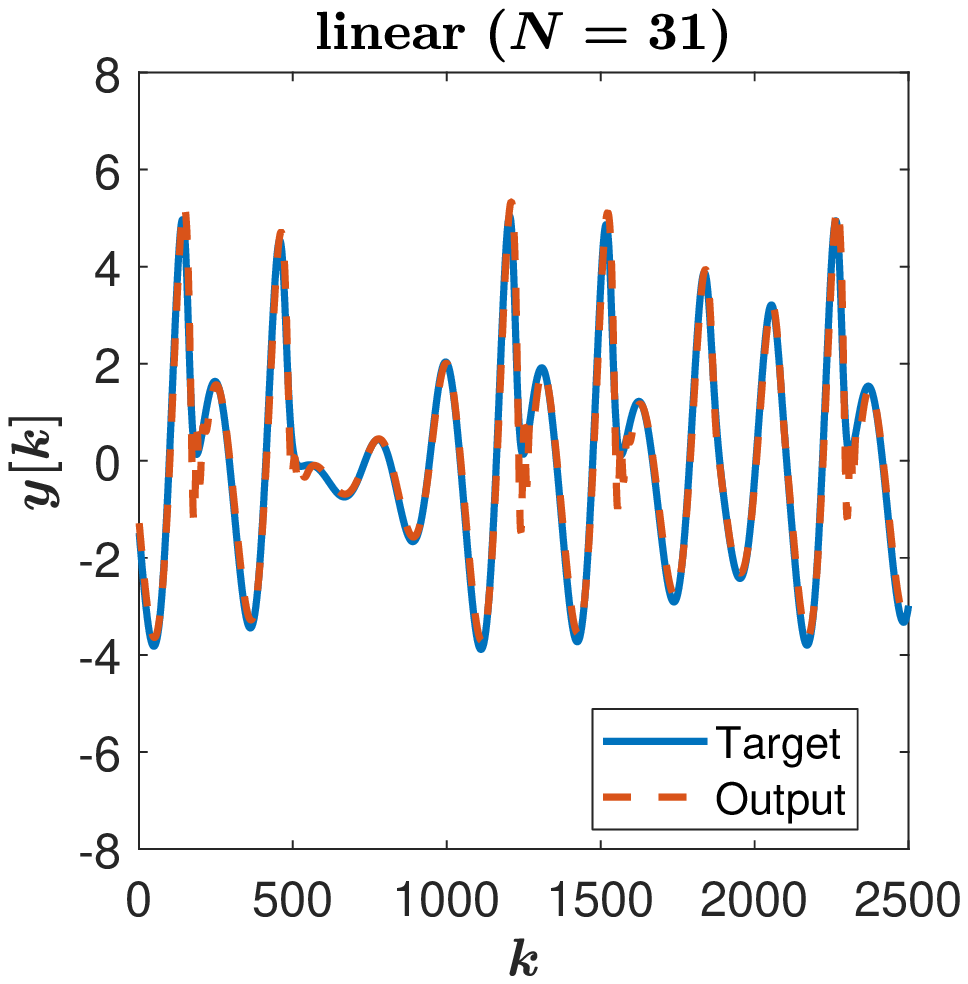} 
  \end{subfigure}

  % \hspace{10pt}
  % \adjustbox{minipage=1em,valign=t}{\subcaption{}\label{fig: k_step_Nmin}}
  % \begin{subfigure}[b]{0.2\linewidth}
  %   \centering\includegraphics[width=\linewidth]{Figure/k_step_Nmin.eps} 
  % \end{subfigure}
  \caption{ (a) Experiment setup for using RCN to predict the time-series generated by a R\"{o}ssler system for $30$-steps ahead. (b) Forecasting the same time-series as (a), using tangent hyperbolic function. (c) The mean squared error (MSE) between the forecast output and target, versus different time-delays, using the same RCN as in (b) with different leakage rate. (d) The mean squared error (MSE) between the forecast output and target, versus different time-delays, using the same RCN as in (b) under various activation functions. (e) Forecasting the same times-series as (a), but changing the activation function into linear function. (f) Forecasting the same time-series as (a) using the reduced RCN computed from (e) through Algorithm \ref{algo:1}.
  }
\end{figure*}

In this example, we show the use of an RCN to learn the temporal evolution of a chaotic system, and analyze how the network configuration affects the performance of the designed RCN. Specifically, we consider the R\"{o}ssler system, given by
\begin{align*}
  dx/dt = -y - z, \quad dy/dt = x+ay, \quad  dz/dt = b + z(x-c),
\end{align*}
where $x,y, z$ are state variables and $a, b,c$ are constant parameters selected as $a = 0.5$, $b = 2.0$, and $c = 4.0$. Initial conditions were selected as $x(0) = 0$, $y(0)=0$, and $z(0) = 1$.

The R\"{o}ssler  system was simulated for $t \in [0,1]$, with $8000$ sampling points collected in this time window. The first $5000$ points were used as the training input, and the sample points between $30$ and $5030$ were used as the reference sequence to train the RCN for $30$-step ahead forecasting. Then, we recorded the output of the RCN for another $3000$ steps as a forecasting sequence. Figure \ref{fig: k_step_Rossler} provides a demonstration of forecasting the $x$-component of R\"{o}ssler system for $30$-steps ahead.

We first randomly generate two RCNs with all hyper-parameters to be the same, but one with $\tanh$ activation function, the other one with linear activation function. The hyper-parameters of the RCN were fixed as follows: number of nodes $N = 500$, leakage rate $\alpha  = 0.8$, length of training data $t_r = 5000$, length of forecasting data $t_s = 3000$, and washout length $w = 500$. If not specifically mentioned, the matrix $A$ was generated under a uniform distribution on $[0,1]$ of sparsity $63.2\%$, and then normalized to have a matrix $2$-norm equals to $0.9$. The output-layer was trained via ridge regression \cite{hoerl1970ridge} with $\lambda = 10^{-8}$.

% \cb{Figure \ref{fig: k_step_tanh} presents the results of forecasting the time-series provided in \ref{fig: k_step_Rossler} with a mean-squared-error (MSE) of $1.49\times 10^{-3}$. Under the same setup as in Figure \ref{fig: k_step_tanh}, we varied the leakage rate $\alpha$ and the activation function of the RCN, and present the corresponding result of MSE v.s. time-delay $\tau$ in Figures \ref{fig: k_step_sweep_alpha} and \ref{fig: k_step_logMSE}, respectively. From Figure \ref{fig: k_step_sweep_alpha}, we observed that the best MSE in most cases was obtained with $\alpha = 0.05$; and from Figure \ref{fig: k_step_logMSE}, we observed that the linear activation achieved similar performance as the nonlinear ones when the time-delay $\tau$ was small. Therefore, we used the same RCN as in Figure \ref{fig: k_step_tanh}, but changed the leakage rate to $\alpha = 0.05$ and the activation function into linear function to forecast the same time-series. The results are presented in Figure \ref{fig: k_step_linear_N=500}, with a MSE of $2.97 \times 10^{-2}$.} 
Figures \ref{fig: k_step_tanh} presents the results of forecasting the time-series provided in Figure \ref{fig: k_step_Rossler} with a mean-squared-error (MSE) of $1.49\times 10^{-3}$. Under the same setup as in Figure \ref{fig: k_step_tanh}, we varied the leakage rate $\alpha$ and the activation function of the RCN, and present the corresponding results of MSE versus time-delay ($\tau$) in Figures \ref{fig: k_step_sweep_alpha} and \ref{fig: k_step_logMSE}, respectively. From Figure \ref{fig: k_step_sweep_alpha}, we observed that $\alpha = 0.05$ resulted in the best MSE across different cases (of $\tau$); and from Figure 2d, we observed that the RCN with linear activation function achieved similar performance as the RCNs with a nonlinear activation function when the time-delay $\tau$ was small. Therefore, we used the same RCN as in Figure \ref{fig: k_step_tanh}, but changed the leakage rate to $\alpha = 0.05$ and the activation function into linear function to forecast the same time-series. The corresponding results are presented in Figure \ref{fig: k_step_linear_N=500} with an MSE of $2.97 \times 10^{-2}$.

We further applied Algorithm \ref{algo:1} on the above RCN with linear activation function, yielding a irreducible linear realization of size $31$ with a MSE of $2.54\times 10^{-2}$. Figure \ref{fig: k_step_linear_N=31} presents the result of using the irreducible RCN to forecast the same time-series generated by R\"{o}ssler system for $30$-steps ahead. Based on our empirical analysis in Section \ref{subsec: RCN design using the irreducible linear realization}, the results in Figure \ref{fig: k_step_linear_N=31} implies that the underlying dynamics of the given input-output sequences can be well-approximated by a linear dynamics with a time-delay of no longer than $31$ steps, which is evident by the experiment setups.

\subsection{Analysis of linear and nonlinear activation functions}

In this part, we further investigate the difference in performance between linear and nonlinear RCNs. The results in this section support our idea of using the linear realization of the RCN to help design a nonlinear RCN with a guaranteed training error, as mentioned in Section \ref{subsec: RCN design using the irreducible linear realization}.

Table \ref{tab:time delay} presents the results of forecasting time-delay systems using linear and nonlinear RCNs. We varied the size of the RCN from $50$ to $500$. For each $N$, we generated a time-delay system with $N$ delay steps (as in Section \ref{subsec: time-delay system}) and simulated $2000$ independent experiments of forecasting $(N-1)$- steps ahead. The training length was fixed as $t_r = 5000$ and the forecast length was fixed as $t_s = 2000$. To make a fair comparison, in each experiment, we generated three RCNs of size $N$, one with linear activation function, one with $\tanh$ activation function, and another with $\mathrm{sigmoid}$ activation function, using the same randomly generated matrices $\tilde{A}$ and $\tilde{B}$. Other hyper-parameters of the three RCNs were set to be the same as in Section \ref{subsec: time-delay system}. The average and standard deviation of training and forecast MSE are reported in Table \ref{tab:time delay}. As we observe from the last column in Table \ref{tab:time delay}, the nonlinear RCNs outperform the linear RCN in terms of the training MSE in most cases.

Table \ref{tab:Rossler} reports the results of forecasting the R\"{o}ssler system using RCNs with different activation function. Using the same dataset as the one in \ref{subsec: Rossler}, we compared the performance of linear and nonlinear RCN at different scales. Similar to the previous table, we varied the size of the RCN from $50$ to $500$ and conducted $2000$ independent experiments to forecast the R\"{o}ssler system for $10$-steps ahead. The training length was fixed as $t_r = 5000$ and the forecast length was fixed as $t_s = 2000$. In each experiment, we generated three RCNs of size $N$, one with linear activation function, one with $\tanh$ activation function, and another with $\mathrm{sigmoid}$ activation function, using the same randomly generated matrices $\tilde{A}$ and $\tilde{B}$. Other hyper-parameters of the three RCNs were set to be the same as in Section \ref{subsec: Rossler}. The average and standard deviation of training and forecast MSE are reported in Table \ref{tab:Rossler}. As we observe from the last column in Table \ref{tab:Rossler}, in this experiment, the training MSE of nonlinear RCNs was always smaller than that of the linear RCNs.

\begin{table*}
\centering
\begin{tabular}{|C{0.5cm}|C{1.5cm}|C{1.5cm}|C{1.5cm}|C{1.5cm}|C{1.5cm}|C{1.5cm}|C{3.5cm}|}
\hline
\multirow{2}{*}{N} & \multicolumn{3}{c}{Training MSE (mean $\pm$ std)} & \multicolumn{3}{|c|}{Forecasting MSE (mean $\pm$ std)} & \multirow{2}{2.2cm}{$\mathbb{P}(\epsilon_{\text{linear}}$ $> \max(\epsilon_{\tanh}, \epsilon_{\mathrm{sigmoid}})$)}\\
\cline{2-7}
& \text{linear} ($\times 10^{-5})$ & $\tanh$ ($\times 10^{-5}$)& $\mathrm{sigmoid}$ ($\times 10^{-7}$) & \text{linear} ($\times 10^{-5})$ & $\tanh$ ($\times 10^{-5}$)& $\mathrm{sigmoid}$ ($\times 10^{-7}$)&\\
\hline
$50$ & $0.11\pm 0.00$ & $0.10\pm 0.00$ & $1.55\pm 0.50$ & $0.28 \pm 0.00$ & $0.24 \pm 0.02$ & $2.53\pm 0.69$ & $99.85\%$  \\
\cline{1-8}
$100$ & $0.04 \pm 0.00$ & $0.04 \pm 0.00$ & $0.53\pm 0.13$ & $0.11 \pm 0.00$ & $0.11 \pm 0.00$ & $1.06\pm 0.25$ & $100\%$\\
\cline{1-8}
$150$ & $3.78 \pm 0.00$ & $3.62 \pm 0.04$ & $0.32\pm 0.06$ & $1.83 \pm 0.00$ & $1.67 \pm 0.03$ & $0.65\pm 0.14$ & $100\%$\\
\cline{1-8}
$200$ & $4.63 \pm 0.00$ & $4.52\pm 0.02$ & $0.24\pm 0.04$ & $11.57 \pm 0.00$ & $11.07 \pm 0.00$ & $0.47\pm 0.08$ & $100\%$\\
\cline{1-8}
$250$ & $0.89 \pm 0.00$ & $0.86 \pm 0.00$ & $0.20\pm 0.03$ & $2.23 \pm 0.00$ & $2.07 \pm 0.02$ & $0.38\pm 0.06$ & $100\%$\\
\cline{1-8}
$300$ & $2.00 \pm 0.00$ & $1.97 \pm 0.00$ & $0.18\pm 0.02$ & $2.38 \pm 0.00$ & $2.37 \pm 0.00$ & $0.32\pm 0.04$ & $100\%$\\
\cline{1-8}
$350$ & $0.02 \pm 0.00$ & $0.02 \pm 0.00$ & $0.16\pm 0.02$ & $0.53 \pm 0.00$ & $0.52 \pm 0.00$ & $0.27\pm 0.03$ & $99.95\%$\\
\cline{1-8}
$400$ & $2.12 \pm 0.00$ & $2.11 \pm 0.00$ & $0.15\pm 0.01$ & $1.40 \pm 0.00$ & $1.37 \pm 0.00$ & $0.25\pm 0.03$ & $99.85\%$\\
\cline{1-8}
$450$ & $4.60 \pm 0.00$ & $4.59 \pm 0.00$ & $0.14\pm 0.01$ & $9.83\pm 0.00$ & $9.41 \pm 0.00$ & $0.23\pm 0.02$ & $100\%$\\
\cline{1-8}
$500$ & $4.45 \pm 0.00$ & $4.36 \pm 0.00$ & $0.13\pm 0.01$ & $11.11 \pm 0.00$ & $10.68 \pm 0.03$ & $0.21\pm 0.02$ & $100\%$ \\
\hline
\end{tabular}
\caption{Results of forecasting time-delay system using linear and nonlinear RCNs. In each experiment, we construct a linear RCN and two nonlinear RCNs (with $\tanh$ and $\mathrm{sigmoid}$ activation functions) of size $N$ using the same hyper-parameters and the same randomly generated $\tilde{A}$ and $\tilde{B}$. The error of using RCNs of size $N$ to forecast a $N$-steps time-delay system is provided in the table, where $2000$ independent experiments are conducted for each $N$.}
    \label{tab:time delay}
\end{table*}

\begin{table*}
\centering
\begin{tabular}{|C{0.5cm}|C{1.5cm}|C{1.5cm}|C{1.5cm}|C{1.5cm}|C{1.5cm}|C{1.5cm}|C{3.5cm}|}
\hline
\multirow{2}{*}{N} & \multicolumn{3}{c}{Training MSE (mean $\pm$ std)} & \multicolumn{3}{|c|}{Forecasting MSE (mean $\pm$ std)} & \multirow{2}{2.2cm}{$\mathbb{P}(\epsilon_{\text{linear}}$ $> \max(\epsilon_{\tanh}, \epsilon_{\mathrm{sigmoid}})$)}\\
\cline{2-7}
& \text{linear} ($\times 10^{-3})$ & $\tanh$ ($\times 10^{-5}$)& $\mathrm{sigmoid}$ ($\times 10^{-5}$) & \text{linear} ($\times 10^{-3})$ & $\tanh$ ($\times 10^{-5}$)& $\mathrm{sigmoid}$ ($\times 10^{-5}$)&\\
\hline
$50$ & $2.96\pm 0.08$ & $8.12\pm 2.16$ & $5.09\pm 0.90 $ & $7.33\pm 0.06$ & $21.6\pm 8.18$ & $5.68\pm 1.18 $ & \multirow{10}{*}{$100\%$}\\
\cline{1-7}
$100$ & $2.88\pm 0.06$ & $3.27\pm 0.46$ & $3.72\pm 0.29$ & $7.39\pm 0.05$ & $4.94\pm 1.61$ & $3.81\pm 0.33 $ &\\
\cline{1-7}
$150$ & $2.85\pm 0.05$ & $2.33\pm 0.29$ & $3.36\pm 0.20 $ & $7.42\pm 0.04$ & $2.83\pm 0.63$ & $0.34\pm 0.23 $ & \\
\cline{1-7}
$200$ & $2.83\pm 0.04$ & $1.86\pm 0.24$ & $3.15\pm 0.15 $ & $7.44\pm 0.03$ & $2.17\pm 0.35$ & $0.32\pm 0.18 $ & \\
\cline{1-7}
$250$ & $2.81\pm 0.04$ & $1.55\pm 0.19$ & $2.99\pm 0.14 $ & $7.45\pm 0.03$ & $1.83\pm 0.24$ & $3.07\pm 0.16 $ & \\
\cline{1-7}
$300$ & $2.80\pm 0.03$ & $1.34\pm 0.17$ & $2.86\pm 0.12 $ & $7.46 \pm 0.03$ & $1.64\pm 0.19$ & $2.93\pm 0.14 $ & \\
\cline{1-7}
$350$ & $2.79\pm 0.03$ & $1.18\pm 0.15$ & $2.75\pm 0.10 $ & $7.47\pm 0.03$ & $1.49\pm 0.15$ & $2.82\pm 0.12 $ & \\
\cline{1-7}
$400$ & $2.78\pm 0.03$ & $1.05\pm 0.14$ & $2.65\pm 0.09 $ & $7.48\pm 0.02$ & $1.39\pm 0.13$ & $2.71\pm 0.12 $ & \\
\cline{1-7}
$450$ & $2.77\pm 0.03$ & $0.96\pm 0.13$ & $2.57\pm 0.09 $ & $7.49\pm 0.02$ & $1.31\pm 0.12$ & $2.62\pm 0.11 $ & \\
\cline{1-7}
$500$ & $2.76\pm 0.02$ & $0.88\pm 0.12$ & $2.48\pm 0.09 $ & $7.49\pm 0.02$ & $1.23\pm 0.10$ & $2.53\pm 0.10 $ & \\
\hline
\end{tabular}
\caption{Results of forecasting R\"{o}ssler system using linear and nonlinear RCNs. In each experiment, we construct a linear RCN and two nonlinear RCNs (with $\tanh$ and $\mathrm{sigmoid}$ activation functions) of size $N$ using the same hyper-parameters and the same randomly generated $\tilde{A}$ and $\tilde{B}$. The error of using RCN of different sizes to forecast the R\"{o}ssler system in Section \ref{subsec: Rossler} for $10$-steps ahead is provided in the table, where $2000$ independent experiments are conducted for each $N$.}
    \label{tab:Rossler}
\end{table*}

%=====================  Sec : Conclusions  =============
\section{Conclusions}\label{sec: conclusion}
In this paper, we provided a detailed analysis and a holistic description of the training procedure and the operation of the RCN. With the help of Takens embedding theorem, we derived the delay embedding map, which an RCN potentially learns during the training process from the given input-output data. This provided insights into the role that the linear activation function and other hyper-parameters play in the design and working of RCNs in applications such as forecasting a time-series. Furthermore, based on the notions of  linear realization theory, we provided a systematic approach to trim RCNs with guaranteed training accuracy. In this context, we introduced the idea of $\alpha$-stable realizations for designing stable RCNs that achieve the desired training objective with reduced size, and established a tractable design algorithm to synthesize RCNs with nonlinear activation. The numerical experiments on forecasting time-delay systems and the R\"{o}ssler system were used to substantiate our proposed design approach for interpretable RCNs. We observed from the experiments that the nonlinear RCNs with both reduced size and guaranteed training accuracy can be attained based on the minimum realization for linear RCNs. These results suggested that the proposed approach offers an informed and interpretable design methodology to devise nonlinear RCNs for a given dataset to decode the underlying dynamics in the data.
%===================================================
\bibliographystyle{IEEEtran}
\bibliography{reference_list.bib}

\begin{thebibliography}{10}
\providecommand{\url}[1]{#1}
\csname url@samestyle\endcsname
\providecommand{\newblock}{\relax}
\providecommand{\bibinfo}[2]{#2}
\providecommand{\BIBentrySTDinterwordspacing}{\spaceskip=0pt\relax}
\providecommand{\BIBentryALTinterwordstretchfactor}{4}
\providecommand{\BIBentryALTinterwordspacing}{\spaceskip=\fontdimen2\font plus
\BIBentryALTinterwordstretchfactor\fontdimen3\font minus
  \fontdimen4\font\relax}
\providecommand{\BIBforeignlanguage}[2]{{%
\expandafter\ifx\csname l@#1\endcsname\relax
\typeout{** WARNING: IEEEtran.bst: No hyphenation pattern has been}%
\typeout{** loaded for the language `#1'. Using the pattern for}%
\typeout{** the default language instead.}%
\else
\language=\csname l@#1\endcsname
\fi
#2}}
\providecommand{\BIBdecl}{\relax}
\BIBdecl

\bibitem{jaeger2001echo}
H.~Jaeger, ``The “echo state” approach to analysing and training recurrent
  neural networks-with an erratum note,'' \emph{Bonn, Germany: German National
  Research Center for Information Technology GMD Technical Report}, vol. 148,
  no.~34, p.~13, 2001.

\bibitem{lukovsevivcius2012reservoir}
M.~Luko{\v{s}}evi{\v{c}}ius, H.~Jaeger, and B.~Schrauwen, ``Reservoir computing
  trends,'' \emph{KI-K{\"u}nstliche Intelligenz}, vol.~26, no.~4, pp. 365--371,
  2012.

\bibitem{TANAKA2019100}
``Recent advances in physical reservoir computing: A review,'' \emph{Neural
  Networks}, vol. 115, pp. 100 -- 123, 2019.

\bibitem{lin2009short}
X.~Lin, Z.~Yang, and Y.~Song, ``Short-term stock price prediction based on echo
  state networks,'' \emph{Expert Systems with Applications}, vol.~36, no. 3,
  Part 2, pp. 7313 -- 7317, 2009.

\bibitem{GRIGORYEVA201459}
``Stochastic nonlinear time series forecasting using time-delay reservoir
  computers: Performance and universality,'' \emph{Neural Networks}, vol.~55,
  pp. 59 -- 71, 2014.

\bibitem{jaeger2004harnessing}
H.~Jaeger and H.~Haas, ``Harnessing nonlinearity: Predicting chaotic systems
  and saving energy in wireless communication,'' \emph{Science}, vol. 304, no.
  5667, pp. 78--80, 2004.

\bibitem{maass2002real}
W.~Maass, T.~Natschl{\"a}ger, and H.~Markram, ``Real-time computing without
  stable states: A new framework for neural computation based on
  perturbations,'' \emph{Neural Computation}, vol.~14, no.~11, pp. 2531--2560,
  2002.

\bibitem{antonelo2008event}
E.~Antonelo, B.~Schrauwen, and D.~Stroobandt, ``Event detection and
  localization for small mobile robots using reservoir computing,''
  \emph{Neural Networks}, vol.~21, no.~6, pp. 862 -- 871, 2008, computational
  and Biological Inspired Neural Networks, selected papers from ICANN 2007.

\bibitem{lukovsevivcius2012practical}
M.~Luko{\v{s}}evi{\v{c}}ius, \emph{A Practical Guide to Applying Echo State
  Networks}.\hskip 1em plus 0.5em minus 0.4em\relax Berlin, Heidelberg:
  Springer Berlin Heidelberg, 2012, pp. 659--686.

\bibitem{doshi2017towards}
F.~Doshi-Velez and B.~Kim, ``Towards a rigorous science of interpretable
  machine learning,'' 2017, https://arxiv.org/abs/1702.08608 [stat.ML].

\bibitem{lipton2018mythos}
Z.~C. Lipton, ``The mythos of model interpretability,'' \emph{Queue}, vol.~16,
  no.~3, pp. 31--57, 2018.

\bibitem{jaeger2007optimization}
H.~Jaeger, M.~Luko{\v{s}}evi{\v{c}}ius, D.~Popovici, and U.~Siewert,
  ``Optimization and applications of echo state networks with leaky-integrator
  neurons,'' \emph{Neural networks}, vol.~20, no.~3, pp. 335--352, 2007.

\bibitem{grigoryeva2018universal}
L.~Grigoryeva and J.-P. Ortega, ``Universal discrete-time reservoir computers
  with stochastic inputs and linear readouts using non-homogeneous state-affine
  systems,'' \emph{J. Mach. Learn. Res.}, vol.~19, no.~1, p. 892–931, Jan.
  2018.

\bibitem{buehner2006tighter}
M.~{Buehner} and P.~{Young}, ``A tighter bound for the echo state property,''
  \emph{IEEE Transactions on Neural Networks}, vol.~17, no.~3, pp. 820--824,
  2006.

\bibitem{yildiz2012re}
I.~B. Yildiz, H.~Jaeger, and S.~J. Kiebel, ``Re-visiting the echo state
  property,'' \emph{Neural Networks}, vol.~35, pp. 1 -- 9, 2012.

\bibitem{jaeger2001short}
H.~Jaeger, \emph{Short term memory in echo state networks}.\hskip 1em plus
  0.5em minus 0.4em\relax GMD-Forschungszentrum Informationstechnik, 2001,
  vol.~5.

\bibitem{Grigoryeva16}
L.~Grigoryeva, J.~Henriques, L.~Larger, and J.-P. Ortega, ``Nonlinear memory
  capacity of parallel time-delay reservoir computers in the processing of
  multidimensional signals,'' \emph{Neural Computation}, vol.~28, no.~7, pp.
  1411--1451, 2016.

\bibitem{Marzen17}
S.~Marzen, ``Difference between memory and prediction in linear recurrent
  networks,'' \emph{Phys. Rev. E}, vol.~96, p. 032308, Sep 2017.

\bibitem{boedecker2012information}
J.~Boedecker, O.~Obst, J.~T. Lizier, N.~M. Mayer, and M.~Asada, ``Information
  processing in echo state networks at the edge of chaos,'' \emph{Theory in
  Biosciences}, vol. 131, no.~3, pp. 205--213, 2012.

\bibitem{rodan2010minimum}
A.~{Rodan} and P.~{Tino}, ``Minimum complexity echo state network,'' \emph{IEEE
  Transactions on Neural Networks}, vol.~22, no.~1, pp. 131--144, 2011.

\bibitem{gallicchio2011architectural}
C.~Gallicchio and A.~Micheli, ``Architectural and markovian factors of echo
  state networks,'' \emph{Neural Networks}, vol.~24, no.~5, pp. 440--456, 2011.

\bibitem{kawai2019small}
Y.~Kawai, J.~Park, and M.~Asada, ``A small-world topology enhances the echo
  state property and signal propagation in reservoir computing,'' \emph{Neural
  Networks}, vol. 112, pp. 15--23, 2019.

\bibitem{zhang2011nonlinear}
B.~{Zhang}, D.~J. {Miller}, and Y.~{Wang}, ``Nonlinear system modeling with
  random matrices: Echo state networks revisited,'' \emph{IEEE Transactions on
  Neural Networks and Learning Systems}, vol.~23, no.~1, pp. 175--182, 2012.

\bibitem{grigoryeva2020dimension}
L.~Grigoryeva and J.-P. Ortega, ``Dimension reduction in recurrent networks by
  canonicalization,'' 2020.

\bibitem{Bollt21}
E.~Bollt, ``On explaining the surprising success of reservoir computing
  forecaster of chaos? the universal machine learning dynamical system with
  contrast to var and dmd,'' \emph{Chaos: An Interdisciplinary Journal of
  Nonlinear Science}, vol.~31, no.~1, p. 013108, 2021.

\bibitem{verzelli2021input}
P.~Verzelli, C.~Alippi, L.~Livi, and P.~Tino, ``Input representation in
  recurrent neural networks dynamics,'' 2021.

\bibitem{Gonon_2020}
\BIBentryALTinterwordspacing
L.~Gonon, L.~Grigoryeva, and J.-P. Ortega, ``Memory and forecasting capacities
  of nonlinear recurrent networks,'' \emph{Physica D: Nonlinear Phenomena},
  vol. 414, p. 132721, Dec 2020. [Online]. Available:
  \url{http://dx.doi.org/10.1016/j.physd.2020.132721}
\BIBentrySTDinterwordspacing

\bibitem{huke1993embedding}
J.~Huke, ``Embedding nonlinear dynamical systems: A guide to takens’ theorem
  (technical report),'' \emph{Manchester Institute for Mathematical Sciences,
  University of Manchester}, 1993.

\bibitem{takens1981detecting}
F.~Takens, ``Detecting strange attractors in turbulence,'' in \emph{Dynamical
  Systems and Turbulence, Warwick 1980}.\hskip 1em plus 0.5em minus 0.4em\relax
  Springer, 1981, pp. 366--381.

\bibitem{grigoryeva2019differentiable}
L.~Grigoryeva and J.-P. Ortega, ``Differentiable reservoir computing.''
  \emph{Journal of Machine Learning Research}, vol.~20, no. 179, pp. 1--62,
  2019.

\bibitem{antonik2018using}
P.~Antonik, M.~Gulina, J.~Pauwels, and S.~Massar, ``Using a reservoir computer
  to learn chaotic attractors, with applications to chaos synchronization and
  cryptography,'' \emph{Phys. Rev. E}, vol.~98, p. 012215, Jul 2018.

\bibitem{carroll2018using}
T.~L. Carroll, ``Using reservoir computers to distinguish chaotic signals,''
  \emph{Phys. Rev. E}, vol.~98, p. 052209, Nov 2018.

\bibitem{verstraeten2010memory}
D.~{Verstraeten}, J.~{Dambre}, X.~{Dutoit}, and B.~{Schrauwen}, ``Memory versus
  non-linearity in reservoirs,'' in \emph{The 2010 International Joint
  Conference on Neural Networks (IJCNN)}, 2010, pp. 1--8.

\bibitem{verzelli2019echo}
P.~Verzelli, C.~Alippi, and L.~Livi, ``Echo state networks with
  self-normalizing activations on the hyper-sphere,'' \emph{Scientific
  Reports}, vol.~9, no.~1, pp. 1--14, 2019.

\bibitem{inubushi2017reservoir}
M.~Inubushi and K.~Yoshimura, ``Reservoir computing beyond memory-nonlinearity
  trade-off,'' \emph{Scientific Reports}, vol.~7, no.~1, p. 10199, 2017.

\bibitem{verstraeten2007experimental}
D.~Verstraeten, B.~Schrauwen, M.~D’Haene, and D.~Stroobandt, ``An
  experimental unification of reservoir computing methods,'' \emph{Neural
  Networks}, vol.~20, no.~3, pp. 391 -- 403, 2007, echo State Networks and
  Liquid State Machines.

\bibitem{de2000minimal}
B.~Schutter, ``Minimal state-space realization in linear system theory: an
  overview,'' \emph{Journal of Computational and Applied Mathematics}, vol.
  121, no.~1, pp. 331 -- 354, 2000.

\bibitem{silverman1971realization}
L.~Silverman, ``Realization of linear dynamical systems,'' \emph{IEEE
  Transactions on Automatic Control}, vol.~16, no.~6, pp. 554--567, 1971.

\bibitem{aras2016new}
S.~Aras and {\.I}.~D. Kocako{\c{c}}, ``A new model selection strategy in time
  series forecasting with artificial neural networks: Ihts,''
  \emph{Neurocomputing}, vol. 174, pp. 974--987, 2016.

\bibitem{brockett2015finite}
R.~W. Brockett, \emph{Finite dimensional linear systems}.\hskip 1em plus 0.5em
  minus 0.4em\relax SIAM, 2015.

\bibitem{hoerl1970ridge}
A.~E. Hoerl and R.~W. Kennard, ``Ridge regression: Biased estimation for
  nonorthogonal problems,'' \emph{Technometrics}, vol.~12, no.~1, pp. 55--67,
  1970.

\end{thebibliography}
\end{document}